\pdfoutput=1
\RequirePackage[T1]{fontenc}
\documentclass[12pt]{article}
\usepackage[skip=1em plus 0.2em minus 0.1em]{parskip}
\usepackage[height=8.85in,width=6.45in]{geometry}

\usepackage[utf8]{inputenc}
\usepackage{amsmath}
\usepackage{amssymb}
\usepackage{mathtools}
\numberwithin{equation}{section}
\usepackage{slashed}
\usepackage{braket}
\usepackage{enumitem}
\usepackage[svgnames]{xcolor}
\usepackage[colorlinks,citecolor=DarkGreen,linkcolor=FireBrick,urlcolor=FireBrick,linktocpage,unicode]{hyperref}
\usepackage{booktabs}
\usepackage{multirow}
\usepackage{adjustbox}
\usepackage{tocloft}

\usepackage[hang,flushmargin,bottom]{footmisc}
\allowdisplaybreaks
\usepackage{graphicx}
\usepackage{tikz}
\usepackage{tikz-cd}
\usepackage{times}
 
\usepackage{bm}
\usepackage{physics}
\usepackage{xcolor}
\usepackage{natbib}
\usepackage{mdframed}
\usepackage{nicefrac}
\usepackage{booktabs}
\usepackage{lipsum}
\usepackage{titlesec}
\usepackage{wrapfig,lipsum,booktabs}
\usepackage{authblk}
\usepackage{blindtext}
\usepackage{subcaption}

\usepackage{algorithm}
\usepackage{algpseudocode}
\newcommand{\commentsymbol}{//}%
\algrenewcommand\algorithmiccomment[1]{\hfill {\footnotesize \commentsymbol{} #1}}

\usepackage{titletoc}
\usepackage{todonotes}
\usepackage{authblk}
\usepackage{setspace}
\usepackage{dsfont} %

\usepackage[nameinlink,capitalize,noabbrev]{cleveref}

\usepackage{CJK}

\usepackage{array}
\usepackage{bbm}
\usepackage{makecell}

\usepackage{dashbox}
\usepackage{xcolor}
\usepackage{colortbl}
\definecolor{lightyellow}{rgb}{1.0, 0.95, 0.7}
\definecolor{Blue}{rgb}{0, 0, 0.8}
\definecolor{blue}{rgb}{0,0,1}
\definecolor{darkgreen}{rgb}{0,0.40,0}
\definecolor{firebrick}{rgb}{0.698,0.133,0.133}

\definecolor{colorA}{rgb}{1,0,0}
\definecolor{colorB}{rgb}{0,0.3,1}
\definecolor{colorC}{rgb}{0.9,0.8,0.2}
\definecolor{colorD}{rgb}{0,0.65,0}
\definecolor{lesslightgray}{rgb}{0.5,0.5,0.5}
\definecolor{light-gray}{gray}{0.95}

\let\hat\widehat

\newcommand{\argmin}{\mathop{\mathrm{argmin}}}  
\newcommand{\argmax}{\mathop{\mathrm{argmax}}}

\def\P{\mathbb{P}}
\def\R{\mathbb{R}}

\DeclareMathOperator*{\E}{{\mathbb{E}}} %

\let\cite\citep 

\usepackage{amsthm}
\usepackage[many]{tcolorbox}
\usepackage{etoolbox}
\BeforeBeginEnvironment{proof}{\par\vspace{-1\baselineskip}}
\AfterEndEnvironment  {proof}{\par\vspace{-1.5\baselineskip}}

\newtheoremstyle{theoremstyle}
  {.5\baselineskip} %
  {.5\baselineskip} %
  {}                  %
  {}                  %
  {\bfseries}        %
  {.}                 %
  {1em}               %
  {}                  %

\theoremstyle{theoremstyle}
\newtheorem{theorem}{Theorem}[section]
\newtheorem{lemma}{Lemma}[section]

\newtheorem{proposition}{Proposition}[section]

\newtheorem{remark}{Remark}[section]

\tcolorboxenvironment{theorem}{
  breakable,
  colback=black!10,
  colframe=white,%
  width=\dimexpr\linewidth+10pt\relax,%
  enlarge left by=-5pt,%
  enlarge right by=-5pt,%
  boxsep=5pt,%
  boxrule=0pt,
  left=0pt,right=0pt,top=0pt,bottom=0pt,
  sharp corners,
  before skip=0.5\baselineskip, %
  after skip=0.5\baselineskip,  %
  fonttitle=\bfseries, %
  coltitle=black %
}
\tcolorboxenvironment{remark}{
  blanker,
  breakable,
  before skip=.8\baselineskip,  %
  after  skip=.8\baselineskip   %
}

\tcolorboxenvironment{proposition}{
  breakable,
  colback=black!10,
  colframe=white,%
  width=\dimexpr\linewidth+10pt\relax,%
  enlarge left by=-5pt,%
  enlarge right by=-5pt,%
  boxsep=5pt,%
  boxrule=0pt,
  left=0pt,right=0pt,top=0pt,bottom=0pt,
  sharp corners,
  before skip=0.5\baselineskip, %
  after skip=0.5\baselineskip,  %
  fonttitle=\bfseries, %
  coltitle=black %
}

\tcolorboxenvironment{lemma}{
  breakable,
  colback=black!10,
  colframe=white,%
  width=\dimexpr\linewidth+10pt\relax,%
  enlarge left by=-5pt,%
  enlarge right by=-5pt,%
  boxsep=5pt,%
  boxrule=0pt,
  left=0pt,right=0pt,top=0pt,bottom=0pt,
  sharp corners,
  before skip=0.5\baselineskip, %
  after skip=0.5\baselineskip,  %
  fonttitle=\bfseries, %
  coltitle=black %
}

\tcolorboxenvironment{corollary}{
  breakable,
  colback=black!10,
  colframe=white,%
  width=\dimexpr\linewidth+10pt\relax,%
  enlarge left by=-5pt,%
  enlarge right by=-5pt,%
  boxsep=5pt,%
  boxrule=0pt,
  left=0pt,right=0pt,top=0pt,bottom=0pt,
  sharp corners,
  before skip=0.5\baselineskip, %
  after skip=0.5\baselineskip,  %
  fonttitle=\bfseries, %
  coltitle=black %
}

\tcolorboxenvironment{definition}{
  breakable,
  colback=black!10,
  colframe=white,%
  width=\dimexpr\linewidth+10pt\relax,%
  enlarge left by=-5pt,%
  enlarge right by=-5pt,%
  boxsep=5pt,%
  boxrule=0pt,
  left=0pt,right=0pt,top=0pt,bottom=0pt,
  sharp corners,
  before skip=0.5\baselineskip, %
  after skip=0.5\baselineskip,  %
  fonttitle=\bfseries, %
  coltitle=black %
}
\tcolorboxenvironment{assumption}{
  breakable,
  colback=black!10,
  colframe=white,%
  width=\dimexpr\linewidth+10pt\relax,%
  enlarge left by=-5pt,%
  enlarge right by=-5pt,%
  boxsep=5pt,%
  boxrule=0pt,
  left=0pt,right=0pt,top=0pt,bottom=0pt,
  sharp corners,
  before skip=0.5\baselineskip, %
  after skip=0.5\baselineskip,  %
  fonttitle=\bfseries, %
  coltitle=black %
}

\crefname{theorem}{Theorem}{Theorems}
\crefname{proposition}{Proposition}{Propositions}
\crefname{lemma}{Lemma}{Lemmas}
\crefname{corollary}{Corollary}{Corollaries}
\crefname{definition}{Definition}{Definitions}
\crefname{assumption}{Assumption}{Assumptions}
\crefname{remark}{Remark}{Remarks}
\crefname{problem}{Problem}{Problems}
\crefname{property}{Property}{property}

\tcolorboxenvironment{hypothesis}{
  breakable,
  colback=black!10,
  colframe=white,%
  width=\dimexpr\linewidth+10pt\relax,%
  enlarge left by=-5pt,%
  enlarge right by=-5pt,%
  boxsep=5pt,%
  boxrule=0pt,
  left=0pt,right=0pt,top=0pt,bottom=0pt,
  sharp corners,
  before skip=0.5\baselineskip, %
  after skip=0.5\baselineskip,  %
  fonttitle=\bfseries, %
  coltitle=black %
}
\crefname{hypothesis}{Hypothesis}{Hypothesises}

\tcolorboxenvironment{fact}{
  breakable,
  colback=black!10,
  colframe=white,%
  width=\dimexpr\linewidth+10pt\relax,%
  enlarge left by=-5pt,%
  enlarge right by=-5pt,%
  boxsep=5pt,%
  boxrule=0pt,
  left=0pt,right=0pt,top=0pt,bottom=0pt,
  sharp corners,
  before skip=0.5\baselineskip, %
  after skip=0.5\baselineskip,  %
  fonttitle=\bfseries, %
  coltitle=black %
}
\crefname{fact}{Fact}{Facts}

\tcolorboxenvironment{example}{
  breakable,
  colback=black!10,
  colframe=white,%
  width=\dimexpr\linewidth+10pt\relax,%
  enlarge left by=-5pt,%
  enlarge right by=-5pt,%
  boxsep=5pt,%
  boxrule=0pt,
  left=0pt,right=0pt,top=0pt,bottom=0pt,
  sharp corners,
  before skip=0.5\baselineskip, %
  after skip=0.5\baselineskip,  %
  fonttitle=\bfseries, %
  coltitle=black %
}
\crefname{example}{Example}{Examples}

\tcolorboxenvironment{question}{
  breakable,
  colback=black!10,
  colframe=white,%
  width=\dimexpr\linewidth+10pt\relax,%
  enlarge left by=-5pt,%
  enlarge right by=-5pt,%
  boxsep=5pt,%
  boxrule=0pt,
  left=0pt,right=0pt,top=0pt,bottom=0pt,
  sharp corners,
  before skip=0.5\baselineskip, %
  after skip=0.5\baselineskip,  %
  fonttitle=\bfseries, %
  coltitle=black %
}

\crefname{question}{Question}{Questions}
\numberwithin{equation}{section}
\numberwithin{theorem}{section}
\numberwithin{proposition}{section}
\numberwithin{definition}{section}
\numberwithin{lemma}{section}
\numberwithin{assumption}{section}
\numberwithin{remark}{section}

\usepackage{lipsum}

\newcommand*{\annot}[1]{\tag*{\footnotesize{\textcolor{black!50}{\big(#1\big)}}}}

\makeatletter
\let\save@mathaccent\mathaccent
\newcommand*\if@single[3]{%
    \setbox0\hbox{${\mathaccent"0362{#1}}^H$}%
    \setbox2\hbox{${\mathaccent"0362{\kern0pt#1}}^H$}%
    \ifdim\ht0=\ht2 #3\else #2\fi
}
\newcommand*\rel@kern[1]{\kern#1\dimexpr\macc@kerna}
\newcommand*\widebar[1]{\@ifnextchar^{{\wide@bar{#1}{0}}}{\wide@bar{#1}{1}}}
\newcommand*\wide@bar[2]{\if@single{#1}{\wide@bar@{#1}{#2}{1}}{\wide@bar@{#1}{#2}{2}}}
\newcommand*\wide@bar@[3]{%
    \begingroup
    \def\mathaccent##1##2{%
        \let\mathaccent\save@mathaccent
        \if#32 \let\macc@nucleus\first@char \fi
        \setbox\z@\hbox{$\macc@style{\macc@nucleus}_{}$}%
        \setbox\tw@\hbox{$\macc@style{\macc@nucleus}{}_{}$}%
        \dimen@\wd\tw@
        \advance\dimen@-\wd\z@
        \divide\dimen@ 3
        \@tempdima\wd\tw@
        \advance\@tempdima-\scriptspace
        \divide\@tempdima 10
        \advance\dimen@-\@tempdima
        \ifdim\dimen@>\z@ \dimen@0pt\fi
        \rel@kern{0.6}\kern-\dimen@
        \if#31
        \overline{\rel@kern{-0.6}\kern\dimen@\macc@nucleus\rel@kern{0.4}\kern\dimen@}%
        \advance\dimen@0.4\dimexpr\macc@kerna
        \let\final@kern#2%
        \ifdim\dimen@<\z@ \let\final@kern1\fi
        \if\final@kern1 \kern-\dimen@\fi
        \else
        \overline{\rel@kern{-0.6}\kern\dimen@#1}%
        \fi
    }%
    \macc@depth\@ne
    \let\math@bgroup\@empty \let\math@egroup\macc@set@skewchar
    \mathsurround\z@ \frozen@everymath{\mathgroup\macc@group\relax}%
    \macc@set@skewchar\relax
    \let\mathaccentV\macc@nested@a
    \if#31
    \macc@nested@a\relax111{#1}%
    \else
    \def\gobble@till@marker##1\endmarker{}%
    \futurelet\first@char\gobble@till@marker#1\endmarker
    \ifcat\noexpand\first@char A\else
    \def\first@char{}%
    \fi
    \macc@nested@a\relax111{\first@char}%
    \fi
    \endgroup
    }
\makeatother
\let\bar\widebar

\definecolor{blue}{named}{black}

\setlist[itemize,enumerate]{
  parsep=\parskip,                                   %
  itemsep=\dimexpr .3em - \parskip\relax plus 2pt,   %
  topsep=\dimexpr 6pt - \parskip\relax plus 1pt minus 1pt,
  partopsep=0pt,
  listparindent=\parindent
}

\begin{document}
\begin{titlepage}

\begin{flushright}
Last Update: \today
\end{flushright}

\vskip 3em
\begin{center}

{
\LARGE \bfseries %
\begin{spacing}{1.15} %
Doob-Based Generative Policy Optimization  %
\end{spacing}
}

\vspace{1.2em}
{\normalsize
\text{
Maojiang Su$^{\dagger,\natural}$
\hspace{0.7em}
Qijie Zhu$^{\ddag}$
\hspace{0.7em}
Yu Xiao
$^{\dagger}$
\hspace{0.7em}
Shuyang Yu$^{\dagger}$
}
\\
\vspace{0.3em}
\text{
Minshuo Chen$^{\S}$
\hspace{0.7em}
Zhaoran Wang$^{\S,\dagger}$
\hspace{0.7em}
Han Liu$^{\dagger,\ddag,\natural}$
}
\\

\vspace{1.8em}
\begin{minipage}{0.75\textwidth}
\raggedright
\small
$^{\dagger\;}$Department of Computer Science,
Northwestern University
\\
$^{\ddag\;}$Department of Statistics and Data Science,
Northwestern University
\\
$^{\S\;}$Department of Industrial Engineering and Management Sciences,
\\
\hphantom{$^{\S\;}$}Northwestern University
\\
$^{\natural\;}$Center for Foundation Models and Generative AI,
Northwestern University
\end{minipage}

\def\thefootnote{*}
\footnotetext{
Project page:
\href{https://maojiangsu.github.io/DoobPO/}
{https://maojiangsu.github.io/DoobPO/}.
}}

\vskip 0.5cm

\end{center}

\noindent
Diffusion and flow models provide expressive policy classes for robotic and sequential decision-making systems.
However, their action likelihoods are generally intractable, limiting the direct use of widely adopted policy optimization methods such as PPO and GRPO. We propose Doob Policy Optimization (DoobPO), a framework that leverages  Doob's $h$-transform for likelihood-free reinforcement learning fine-tuning of generative policies. DoobPO separates policy improvement from generative policy learning. It first learns a lightweight ratio model that represents the desired policy update, and then distills the resulting ratio-tilted target policy into the generative policy via conditional drift matching. This formulation applies uniformly to both diffusion and flow policies and supports a broad family of policy improvement rules, including PPO and GRPO. 
We prove that the practical conditional drift matching objective is gradient-equivalent to the ideal Doob drift matching objective. 
We further show that the induced diffusion process solves an endpoint-constrained stochastic optimal control problem. This establishes that it achieves the minimum-energy modification required to realize the desired policy tilt.
Across MuJoCo continuous-control benchmarks, DoobPO achieves strong performance against competitive diffusion- and flow-policy baselines,  showing a principled and practical route to RL fine-tuning for generative policies.

\end{titlepage}

\section{Introduction}
\label{sec:intro}

Generative policies based on diffusion models
\cite{song2019generative,song2020score,ho2020denoising}
and flow matching models
\cite{lipman2022flow,liu2022flow}
have emerged as a powerful paradigm for policy learning.
Diffusion policies generate actions through iterative denoising, while flow policies generate actions by integrating a learned velocity field, often providing a deterministic and more efficient alternative.
Due to their flexibility and accuracy in modeling high-dimensional action distributions, generative policies have achieved strong empirical gains over deterministic or unimodal policies in challenging continuous-control settings
\cite{ke20243d,chi2025diffusion,huang2024diffuseloco,black2024pi,intelligence2025pi}.

A natural way to further improve generative policies is through reinforcement learning fine-tuning.
However, a central challenge is that the exact likelihoods of generative policies are generally intractable, since each action is produced only after simulating an entire denoising or flow trajectory.
This prevents the direct use of many stable and widely adopted policy optimization methods, such as PPO
\cite{schulman2017proximal}
and GRPO
\cite{shao2024deepseekmath}.
Several recent works have attempted to bypass this likelihood challenge
\cite{psenka2023learning,ding2024diffusion,ma2025efficient,ren2024diffusion,zhang2025reinflow,mcallister2025flow}.
However, these methods often suffer from biased estimation, limited data efficiency, substantial computational overhead, or restriction to a particular generative policy class or policy update rule.
Consequently, the full potential of RL for diffusion and flow policies remains underexplored.

To address this problem, we propose \textbf{Doob Policy Optimization (DoobPO)}, a likelihood-free framework that separates policy improvement from generative policy learning.
The first stage represents the desired policy improvement through a target policy ratio.
The second stage distills the resulting ratio-tilted target into the generative policy.
This decomposition avoids evaluating the likelihood of the generative policy itself.
The policy-improvement stage is agnostic to the underlying generative model and accommodates a broad family of updates, including regularized Policy Mirror Descent~\cite{zhan2023policy}, Soft Actor-Critic~\cite{haarnoja2018soft}, PPO~\cite{schulman2017proximal}, and GRPO~\cite{shao2024deepseekmath}.
The generative stage uses Doob's $h$-transform to characterize the ratio-tilted target process and learns this process through Doob's Drift Matching.
The same construction applies to both diffusion and flow policies.

\textbf{Contributions.}
Our contributions are three-fold:
\begin{itemize}
    \item \textbf{Method.}
    We propose \textbf{Doob Policy Optimization (DoobPO)}, a unified framework for RL fine-tuning of generative policies without evaluating exact action likelihoods
    (\cref{sec:method}).
    DoobPO follows a two-stage design: it first obtains the target policy ratio and then transfers the resulting policy improvement into the generative policy through a tractable conditional drift matching objective.
    The framework supports both diffusion and flow policies
    (\cref{sec:diff_dr_matching,sec:flow_dr_matching}).
    We summarize
    the procedures in
    \cref{alg:rwdm_ppo,alg:rwdm_ppo_fm}.
    Instantiating DoobPO with different policy updates yields
    \textbf{DoobPMD}, \textbf{DoobPPO}, and \textbf{DoobGRPO}.
    We further show that several existing methods arise as special cases under particular choices of the policy update and generative policy class
    (\cref{sec:unified_framework}).
    We summarize these instantiations in \Cref{tab:unified_framework}.

    \item \textbf{Theory.}
    We characterize the ratio-tilted target policy through Doob's $h$-transform
    (\cref{lem:algo_correctness}).
    This characterization leads to Doob's Drift Matching as a principled way to realize the desired policy update.
    For tractability, we introduce a conditional drift matching objective and prove that its gradient is proportional to that of the ideal Doob's Drift Matching objective for both diffusion and flow policies
    (\cref{thm:grad_equi_diff,thm:grad_equi_fm}).
    From the Stochastic Optimal Control perspective
    (\cref{sec:soc}),
    we further show that the Doob-transformed diffusion solves the corresponding endpoint-constrained control problem
    (\cref{prop:cdm_solves_soc}).
    Hence, DoobPO learns the minimum-energy modification of the pretrained diffusion process that realizes the terminal tilt
    (\cref{rem:ratio_h_interpretation}).

    \item \textbf{Experiments.}
    We evaluate DoobPO on MuJoCo continuous-control benchmarks with both diffusion and flow policies
    (\cref{sec:exp}).
    For diffusion policies, DoobPO achieves the strongest overall performance against competitive diffusion-policy RL baselines, with leading results on most tasks
    (\cref{sec:diffusion_exp}).
    For flow policies, DoobGRPO achieves the best performance on 9 out of 10 tasks under the same interaction budget and delivers large gains on challenging locomotion environments
    (\cref{sec:flow_exp}).
    These results establish DoobPO as a strong and practical framework for RL fine-tuning across both diffusion and flow policies.
\end{itemize}

\textbf{Notation.}
We summarize notation and time conventions in \cref{sec:notation}.

\textbf{Related Work.}
We defer an extended discussion of related work to \cref{sec:related_works} due to page limits.

\section{Preliminaries}
\label{sec:background}

\textbf{Reinforcement Learning.}
We formulate the task by defining a Markov Decision Process (MDP) \cite{sutton1998reinforcement} with a tuple $(\mathcal{S}, \mathcal{A}, P, R, \gamma)$.
Here, $\mathcal{S}$ and $\mathcal{A}$ denote the state and action spaces.
We use $o \in \mathcal{S}$ to denote a state and $a \in \mathcal{A}$ to denote an action.
The transition kernel is denoted by $P(o' | o,a)$, the reward function by $R(o,a)$,
and $\gamma \in [0,1)$ is the discount factor.
A policy $\pi(a | o)$ induces a distribution over trajectories
$\tau = (o_0,a_0,o_1,a_1,\ldots)$, where $a_t \sim \pi(\cdot | o_t)$
and $o_{t+1} \sim P(\cdot | o_t,a_t)$.
The standard objective of RL is to maximize
\begin{align}
\label{eq:rl_objective_prelim}
    J(\theta)
    =
    \E_{\tau \sim d_{\pi_\theta}(\tau)}
    \Big[
    \sum_{t \ge 0} \gamma^t R(o_t,a_t)
\Big],
\end{align}
where $d_{\pi_\theta}(\tau)$ is the distribution of trajectories generated by the policy $\pi_\theta$ and the transition kernel.
Policy gradient methods (e.g., REINFORCE \cite{williams1992simple}) update $\theta$ by ascending an estimate of $\nabla_\theta J(\theta)$,
\begin{align}
\label{eq:policy_gradient}
    \nabla_\theta J(\theta)
    =
    \E_{\tau \sim d_{\pi_\theta}(\tau)}
    \Big[\sum_{t\ge 0} \gamma^t
        \nabla_\theta \log \pi_\theta(a_t | o_t)
        \sum_{k\ge t}\gamma^{k-t} R(o_k,a_k)
    \Big],
\end{align}
where $\nabla_\theta \log \pi_\theta(a_t | o_t)$ is the score function of the policy.
To reduce variance, we often replace the cumulative future reward $\sum_{k\ge t}\gamma^{k-t} R(o_k,a_k)$ with an advantage estimate $A(o_t,a_t)$.

In practice, directly optimizing \eqref{eq:policy_gradient}
may be unstable due to overly large policy updates.
Trust-region style policy gradient methods \cite{schulman2015trust,schulman2017proximal,shao2024deepseekmath} instead constrain the policy change between successive iterates.
A popular example is Proximal Policy Optimization (PPO) \cite{schulman2017proximal}, which maximizes $J_{\mathrm{PPO}}(\theta)$
\begin{align}
    \E_{\tau \sim d_{\pi_{\text{old}}}(\tau)}
    \Big[
        \sum_{t\ge 0}
        \min\Big(
            r_\theta(o_t,a_t)\hat{A}_{\mathrm{old}}(o_t,a_t),
            \mathrm{clip}\big(r_\theta(o_t,a_t),1-\epsilon,
            1+\epsilon\big)\hat{A}_{\mathrm{old}}(o_t,a_t)
        \Big)
    \Big],
    \label{eq:ppo}
\end{align}
where $\pi_{\mathrm{old}}$ is the behavior policy, $r_\theta(o,a):=\pi_\theta(a|o)/\pi_{\mathrm{old}}(a|o)$ is the probability ratio, $\epsilon>0$ is the clipping parameter, and $\hat{A}_{\mathrm{old}}(o,a)$ is an advantage estimator.
The clipped surrogate upweights actions that are estimated to be better than baseline and limits the policy update away from $\pi_{\mathrm{old}}$.

\begin{table*}[t]
\centering
\caption{\small Doob Policy Optimization}
\label{tab:unified_framework}
\resizebox{0.8\linewidth}{!}{%
\begin{tabular}{l|ccc}
\toprule
Method
& Generative policy
& Policy update method
& Ratio
\\
\midrule
DPMD \cite{ma2025efficient}
& Diffusion
& Policy Mirror Descent (KL-proximal)
& Analytic
\\

SDAC \cite{ma2025efficient}
& Diffusion
& Soft Actor Critic
& Analytic
\\

DoobPMD (ours)
& Flow
& Policy Mirror Descent (KL-proximal)
& Analytic
\\

DoobPPO (ours)
& Diffusion / Flow
& Proximal Policy Optimization
& Learned
\\

DoobGRPO (ours)
& Diffusion / Flow
& Group Relative Policy Optimization
& Learned
\\
\bottomrule
\end{tabular}}
\end{table*}

\textbf{Diffusion Models and Score Matching.}
Let $p_0$ denote the data distribution.
Diffusion models define a forward noising process
$(Y_t)_{t \in [0,T]}$ that gradually transforms data
into noise through the Stochastic Differential Equation
(SDE) \cite{song2020score}
\begin{align}
\label{eq:forward_sde}
    \dd Y_t
    =
    f(Y_t,t)\dd t
    +
    \sigma_t\dd W_t,
    \quad
    Y_0 \sim p_0,
    \quad
    t \in [0,T],
\end{align}
where $f$ and $\sigma_t$ are pre-designed drift and
diffusion coefficients, and $W_t$ is a standard
Brownian motion.
We denote by $p_t$ the marginal distribution of $Y_t$.
The generative process is the reverse-time denoising
process $(X_t)_{t \in [0,T]}$, evolving from $t=T$
to $t=0$.
It transports noise back to data, with
$X_T \sim p_T$ and $X_0 \sim p_0$.
Its time-$t$ marginal is $p_t$, and the process
is governed by the reverse-time SDE
\begin{align}
\label{eq:reverse_sde}
    \dd X_t
    =
    \bigl[
        f(X_t,t)
        -
        \sigma_t^2 \nabla_x \log p_t(X_t)
    \bigr]\dd t
    +
    \sigma_t \dd \overline W_t,
    \quad
    X_T \sim p_T,
    \quad
    t:T \to 0.
\end{align}
Here, $\overline W_t$ denotes a Wiener process in
the reverse-time direction.

Directly using \eqref{eq:reverse_sde} is difficult
in practice because the noise distribution $p_T$
and the true score $\nabla_x \log p_t$ are not
known explicitly.
Therefore, \eqref{eq:reverse_sde} is implemented
through two approximations.
First, the noise distribution $p_T$ is approximated
by a standard Gaussian distribution.
Second, the ground-truth score function
$s(x,t) := \nabla_x \log p_t(x)$
is replaced by a parameterized estimator
$s_\phi(x,t)$.
The score estimator is learned by minimizing
the score matching loss
\begin{align}
\label{eq:score_matching}
    \mathcal{L}_{\mathrm{SM}}(\phi)
    =
    \E_{t, Y_t \sim p_t}
    \bigl[
        \|s(Y_t,t) - s_\phi(Y_t,t)\|_2^2
    \bigr].
\end{align}

\textbf{Flow Matching.}
Following the notation in \cite{lipman2024flow}, let $p_0$ denote the noise distribution and $p_T$ the data distribution (opposite to the diffusion convention).
Flow matching constructs a deterministic transport from $p_0$ to $p_T$ via the Ordinary Differential Equation (ODE)
\begin{align}
\label{eq:fm_ode}
    \dd X_t = u(X_t,t)\dd t,
\end{align}
where $u(x,t)$ is the velocity field and $p_t$ denotes the marginal probability density function of $X_t$.
Given a reference velocity $u(x,t)$ that induces the desired path $\{p_t\}_{t\in[0,T]}$, flow matching learns a parameterized velocity $u_\phi(x,t)$ by minimizing
the flow matching loss
\begin{align}
\label{eq:flow_matching}
    \mathcal{L}_{\mathrm{FM}}(\phi)
    =
    \E_{t, X_t \sim p_t}
    \Big[\big\| u(X_t,t) - u_\phi(X_t,t)\big\|_2^2\Big].
\end{align}
After training, new samples are generated by simulating the ODE \eqref{eq:fm_ode} from $t=0$ to $t=T$, replacing the unknown velocity $u(x,t)$ with its learned estimator $u_\phi(x,t)$.

\textbf{Diffusion Policy and Flow Policy.}
A diffusion policy \cite{chi2025diffusion} is a policy induced by a diffusion model.
Let $s_\phi(x,t)$ be a parameterized score network, and let
$\{p_t^\phi\}_{t\in[0,T]}$ denote the marginal distributions obtained by
simulating the reverse-time SDE \eqref{eq:reverse_sde} with the score
$\nabla_x \log p_t(x)$ replaced by $s_\phi(x,t)$.
The diffusion policy $\pi_\phi$ is defined as the distribution over actions induced at the terminal time of the denoising process,
$\pi_\phi(\cdot|o) = p_0^\phi(\cdot|o)$.
A flow policy \cite{black2024pi,intelligence2025pi} is a policy induced by a flow matching model.
Let $u_\phi(x,t)$ be a parameterized velocity field, and let
$\{p_t^\phi\}_{t\in[0,T]}$ denote the marginals obtained by simulating the ODE \eqref{eq:fm_ode} with $u(x,t)$ replaced by $u_\phi(x,t)$.
The flow policy $\pi_\phi$ is defined as the action distribution at the terminal time of the flow, i.e.,
$\pi_\phi(\cdot|o)=p_T^\phi(\cdot|o)$.

For both diffusion and flow policies, evaluating the likelihood $\pi_\phi(a|o)$ is expensive.
As a result, the probability ratio
$\pi_\phi(a|o)/\pi_{\mathrm{old}}(a|o)$ in trust-region policy gradient methods (e.g., PPO, GRPO) is intractable.
This limits applying these methods to improve diffusion and flow policies.
DoobPO overcomes this obstacle and enables stable policy gradient style optimization for generative policies.

\section{Likelihood-Free Generative Policy Optimization}
\label{sec:method}

We introduce \textbf{Doob Policy Optimization (DoobPO)}, a general framework for RL fine-tuning of diffusion policies (\cref{sec:diff_dr_matching}) and flow policies (\cref{sec:flow_dr_matching}) without evaluating exact action likelihoods.
DoobPO supports a broad family of trust-region policy gradient methods, while \cref{sec:unified_framework} shows how both existing and proposed algorithms arise under the same framework.

\subsection{Diffusion Policy Optimization with DoobPO}
\label{sec:diff_dr_matching}

In this section, we introduce DoobPO for diffusion policies.
We start from a base diffusion model with score network $s_\phi(x,t)$, which induces a diffusion policy $\pi_\phi$.
Our goal is to optimize the policy objective \eqref{eq:rl_objective_prelim} over parameters $\theta$, starting from the pretrained initialization $\theta=\phi$.

\textbf{Ratio Learning.}
We first learn a probability ratio that captures the policy improvement signal induced by trust-region policy-gradient updates, including regularized Policy Mirror Descent \cite{zhan2023policy}, Soft Actor-Critic \cite{haarnoja2018soft}, PPO \cite{schulman2017proximal}, and GRPO \cite{shao2024deepseekmath}.
Specifically, we introduce a lightweight ratio network $r_\beta(o,a)$ to approximate the intractable likelihood ratio
$r_\theta(o,a)=\pi_\theta(a|o)/\pi_{\mathrm{old}}(a|o)$.
We train $r_\beta$ on trajectories collected under the behavior policy $\pi_{\mathrm{old}}$ by optimizing a clipped reinforcement learning objective $J_{\mathrm{RL}}$.
For example, under PPO, we use
$J_{\mathrm{RL}} = J_{\mathrm{PPO}}$
following \eqref{eq:ppo}, with $r_\theta$ replaced by $r_\beta$.
Intuitively, $r_\beta(o,a)>1$ upweights an action relative to $\pi_{\mathrm{old}}$, while $r_\beta(o,a)<1$ downweights it.
The clipping term enforces a conservative trust-region-style update.
In practice, we parameterize the ratio model as
$r_\beta(o,a)=\exp(\Tilde{r}_\beta(o,a))$ to ensure $r_\beta(o,a)>0$.

\textbf{Doob's Drift Matching.}
We now distill the policy improvement information encoded by the learned ratio $r_\beta$ into the score model.
We consider a diffusion policy, where the denoising SDE \eqref{eq:reverse_sde} is conditioned on the environment state $o$, and the score estimator takes the form $s_\phi(x,t,o)$.
This score model induces the behavior policy $\pi_{\mathrm{old}}(\cdot|o)$ as the terminal distribution of the denoising process.

Given the learned ratio $r_\beta(o,a)$, we would like to update the policy toward the ratio-tilted distribution
$\pi(\cdot|o)\propto r_\beta(o,\cdot)\pi_{\mathrm{old}}(\cdot|o)$.
We realize this tilt by modifying the denoising dynamics with an additional drift term.
We characterize this modification via Doob's $h$-transform.

\begin{lemma}[Doob's $h$-Transform]
\label{lem:algo_correctness}
Fix a state $o \in\mathcal{S}$.
Let $(X_t)_{t\in[0,T]}$ follow the reverse-time denoising SDE \eqref{eq:reverse_sde} with $\nabla_x \log p_t(x|o)$ replaced by its estimator $s_\phi(x,t,o)$, inducing conditional marginals
$\{p_t^\phi(\cdot|o)\}_{t\in[0,T]}$.
Let $h_0(\cdot,o):\mathcal{A}\to\R_+$ be a positive integrable function.
Let $h(\cdot,o,\cdot):\mathcal{A}\times[0,T]\to \R_+$ be the associated Doob $h$-function, denoted by $h_t(x,o)$, defined as
\begin{align}
\label{eq:solution_exp}
    h_t(x,o)
    =
    \E \big[h_0(X_0,o)| X_t=x,o\big],
\end{align}
where the conditional expectation is taken under the original (untilted) denoising process induced by $s_\phi$.
The Doob $h$-transformed denoising process starts from
\[
    X_T^h \sim p_T^h(\cdot|o),
    \qquad
    p_T^h(x|o)
    :=
    \frac{
        p_T^\phi(x|o) h_T(x,o)
    }{
        \E_{X_0\sim p_0^{\phi}(\cdot|o)}[h_0(X_0,o)]
    },
\]
and evolves backward from $T$ to $0$ according to
\begin{align}
\label{eq:new_process}
    \dd X_t^{h}
    =
    \Big[
        f(X_t^{h},t)
        -\sigma_t^2 s_\phi(X_t^{h},t,o)
        -\sigma_t^2 \nabla_x \log h_t(X_t^{h},o)
    \Big]\dd t
    +\sigma_t\dd W_t.
\end{align}
Then its conditional marginals satisfy
\begin{align}
    p_t^{h}(x|o)
    =
    \frac{
        p_t^\phi(x|o)h_t(x,o)
    }{
        \E_{X_0\sim p_0^{\phi}(\cdot|o)}[h_0(X_0,o)]
    },
    \quad t\in[0,T].
    \label{eq:p_t^h}
\end{align}
Particularly, taking $h_0(x,o)=r_\beta(o,x)$ yields
$p_0^h(x|o)\propto p_0^{\phi}(x|o)r_\beta(o,x)$.
With $p_0^{\phi}(\cdot|o)=\pi_{\mathrm{old}}(\cdot|o)$, the time-$0$ policy induced by the tilted denoising process satisfies
$\pi(a|o)\propto r_\beta(o,a)\pi_{\mathrm{old}}(a|o)$.
\end{lemma}

\begin{proof}
Please see \cref{sec:proof_corr} for a detailed proof.
\end{proof}

\begin{remark}[Initialization]
In \cref{lem:algo_correctness}, the Doob-transformed denoising process \eqref{eq:new_process} is initialized from
$p_T^h(\cdot|o)$, which does not necessarily coincide with the original terminal law
$p_T^\phi(\cdot|o)$.
Under the idealized fully-noised setting, where $X_T$ is independent of $X_0$ conditional on $o$, we have
\[
    h_T(x,o)
    =
    \mathbb E[h_0(X_0,o)|X_T=x,o]
    =
    \mathbb E_{X_0\sim p_0^\phi(\cdot|o)}[h_0(X_0,o)].
\]
Therefore, $p_T^h(\cdot|o)=p_T^\phi(\cdot|o)$ under this assumption.
In practice, we initialize \eqref{eq:new_process} from the same standard Gaussian prior as the reference process.
\end{remark}

We then learn a parameterized conditional score model $s_\theta(x,t,o)$ to match the ratio-induced drift $\nabla_x \log h_t(x,o)$ via the Doob's Drift Matching objective
\begin{align}
\label{eq:def_DM}
    \mathcal{L}_{\mathrm{DM}}(\theta)
    =
    \E_{t, o,X_t \sim p_t^{h}(\cdot|o)}
    \Big[
        \big\|
            \nabla_x \log h_t(X_t,o)
            - s_\theta(X_t,t,o)
        \big\|_2^2
    \Big],
\end{align}
where $\theta$ denotes the learnable parameters and $t$ samples from uniform distribution $U[0,T]$.

Although the Doob's Drift Matching (DM) objective is clean and direct, it requires access to both $\nabla_x \log h_t(x,o)$
and samples from the tilted distribution $p_t^{h}(\cdot|o)$, which are generally intractable.
To address this issue, we introduce a conditional objective that attains the same optimum as $\mathcal{L}_{\mathrm{DM}}$.
Specifically, we define the Conditional Drift Matching (CDM) objective
\begin{align}
\label{eq:def_cDM}
    \mathcal{L}_{\mathrm{CDM}}(\theta)
    =
    \E_{t,o, X_t \sim p_t^\phi, X_0 \sim p_{0| t}^\phi}
    \Big[
        r_\beta(o,X_0)
        \big\|
            \nabla_x \log p_{0|t}^\phi (X_0| X_t,o)
            - s_\theta(X_t,t,o)
        \big\|_2^2
    \Big],
\end{align}
where $p_{0| t}^\phi(\cdot| X_t,o)$ denotes the diffusion conditional distribution.
More importantly, for a fixed state $o$, we show that the objectives in \eqref{eq:def_DM} and \eqref{eq:def_cDM} share the same gradient direction.

\begin{theorem}[Gradient Equivalence for Diffusion]
\label{thm:grad_equi_diff}
Fix $o \in\mathcal S$, define normalizing constant
$Z_{\beta}(o) := \E_{X_0 \sim p_0^\phi(\cdot|o)}\big[r_\beta(o,X_0)\big]$ and
\begin{align*}
    \ell_{\mathrm{DM}}(\theta,o)
    &:=
    \E_{t,X_t\sim p_t^h(\cdot|o)}
    \left[
        \|
        \nabla_x\log h_t(X_t,o)-s_\theta(X_t,t,o)
        \|_2^2
    \right],
    \\
    \ell_{\mathrm{CDM}}(\theta,o)
    &:=
    \E_{t,X_t\sim p_t^\phi(\cdot|o),X_0\sim p_{0|t}^\phi(\cdot|X_t,o)}
    \left[
        r_\beta(o,X_0)
        \|
        \nabla_x\log p_{0|t}^\phi(X_0|X_t,o)
        -
        s_\theta(X_t,t,o)
        \|_2^2
    \right].
\end{align*}
Then their gradients satisfy
$\nabla_\theta \ell_{\mathrm{CDM}}(\theta,o)
=
Z_{\beta}(o) \cdot \nabla_\theta \ell_{\mathrm{DM}}(\theta,o)$.
\end{theorem}

\begin{proof}
Please see \cref{sec:proof_grad_equi} for a detailed proof.
\end{proof}

\cref{thm:grad_equi_diff} is stated for each fixed observation $o$.
After averaging over observations, we obtain
\begin{align*}
    \nabla_\theta \mathcal L_{\mathrm{CDM}}(\theta)
    =
    \nabla_\theta \E_o \left[Z_{\beta}(o)\ell_{\mathrm{DM}}(\theta,o)\right].
\end{align*}
Therefore, $\mathcal L_{\mathrm{CDM}}$ has the same gradient field as the reweighted DM objective
$\E_o[Z_{\beta}(o)\ell_{\mathrm{DM}}(\theta,o)]$.
If the model class $\{s_\theta:\theta\in\Theta\}$ is realizable and the minimum of $\mathcal L_{\mathrm{CDM}}$ is attained, this reweighting does not change the global minimizer.
Indeed, there exists $\theta^\star\in\Theta$ such that
$s_{\theta^\star}(x,t,o)=\nabla_x\log h_t(x,o)$
for every $x\in\mathcal A$, $t\in[0,T]$, and $o\in\mathcal S$.
Hence $\ell_{\mathrm{DM}}(\theta^\star,o)=0$ for every $o\in\mathcal S$,
so $\theta^\star$ also minimizes both the original DM objective and the reweighted DM objective, which implies
$\theta^\star$ also minimizes $\mathcal L_{\mathrm{CDM}}$.

For variance-preserving DDPMs, the forward noising
process \eqref{eq:forward_sde} has the closed-form
representation
$X_t = \sqrt{\bar\alpha_t}X_0
    + \sqrt{1-\bar\alpha_t}Z,
    \quad Z\sim\mathcal N(0,I)$,
where
$\bar\alpha_t
= \exp(-\int_0^t \sigma_s^2 \dd s)$
is the cumulative signal retention coefficient.
Moreover, the inference process follows the backward SDE \eqref{eq:reverse_sde} by replacing the drift term by $-\sigma_t^2 x/2 - \sigma_t^2 s_\phi(x,t)$.
Therefore, following the Bayes formula $p_{0|t} = p_{t|0} p_0/p_t$, the CDM objective takes the form
\begin{align}
    \mathcal{L}_{\text{CDM}}(\theta)
    = & ~ \E_{t,o, X_t, X_0}
    [r_\beta(o,X_0) \|\nabla \log p_{t|0}^\phi(X_t|X_0,o) - (s_\phi + s_\theta)(X_t,t,o) \|_2^2] \nonumber \\
    \approx & ~\E_{t,o, X_t, X_0}
    [ r_\beta(o,X_0) \| -\frac{X_t - \sqrt{\bar{\alpha}_t}X_0}{1-\bar{\alpha}_t}- (s_\phi + s_\theta)(X_t,t,o) \|_2^2],
    \label{eq:cdm_loss_ddpm}
\end{align}
where $X_t \sim p_t$, $t \sim U[0,T]$, and $X_0$ is sampled from the conditional distribution
$p_{0|t}(\cdot| X_t)$.
Let the score model $s_\phi$ be the pre-trained base model that fixes the noise schedule and the corresponding regression target.
The approximation in the second line replaces the conditional score
$\nabla \log p_{t|0}^\phi(X_t|X_0,o)$ by the Gaussian regression target induced by the forward noising kernel.
The approximation is reliable when $s_\phi$ is well-trained under the chosen diffusion construction and the trust-region update keeps $\pi_{\mathrm{old}}$ close to the base behavior.

\textbf{Drift Parametrization.}
The conditional drift matching loss \eqref{eq:cdm_loss_ddpm} requires evaluating
the base score network $s_\phi$.
In the advantage-tilted denoising process \eqref{eq:new_process}, both the base
score $s_\phi$ and the advantage-induced drift $\nabla_x \log h_t$ appear with
the same coefficient $\sigma_t^2$.
We exploit this structure and adopt a simplified parametrization.
Specifically, we parametrize the drift model $g_\theta$ to match the sum
$\nabla_x \log h_t(x) + s_\phi(x,t)$.
Then the conditional drift matching loss
\eqref{eq:cdm_loss_ddpm} reduces to
\begin{align}
    \mathcal{L}_{CDM}(\theta)
    = \E_{t,o, X_t, X_0}
    \Big[r_\beta(o,X_0)
    \| -\frac{X_t - \sqrt{\bar{\alpha}_t}X_0}{1-\bar{\alpha}_t} - g_\theta(X_t,t,o) \|_2^2 \Big],
    \label{eq:cdm_loss_prac}
\end{align}
which no longer requires evaluating the base score network $s_\phi$ during
training.
This parametrization also simplifies inference.
Once we obtain the trained score network $g_\theta$, sampling from the target distribution $p_0^h$ only requires simulating the
following SDE:
\begin{align}
\label{eq:new_process_prac}
    \dd X_t^h
    = [f(X_t^h,t)
    - \sigma_t^2 g_\theta(X_t^h,t,o)
    ] \dd t + \sigma_t \dd W_t.
\end{align}
We summarize the training procedure in \cref{alg:rwdm_ppo}.

\subsection{Flow Policy Optimization with DoobPO}
\label{sec:flow_dr_matching}

In this section, we instantiate DoobPO for flow policies.
We start from a pre-trained flow matching model $u_\phi(x,t,o)$, which induces a flow policy $\pi_\phi(\cdot|o)$ with ODE \eqref{eq:fm_ode}.
Our goal is to solve \eqref{eq:rl_objective_prelim}.
We reuse the same ratio learning stage as in \cref{sec:diff_dr_matching}.

\textbf{Doob's Drift Matching.}
We then distill the policy improvement signal encoded by $r_\beta$ into the velocity model.
For flow matching, the policy is given by the terminal distribution at time $T$.
We consider the joint distribution $p_{t,T}^\phi$
induced by the independent-coupling affine interpolation
$X_t=\alpha_t X_0+\beta_t X_T$, where
$X_0\sim\mathcal N(0,I)$,
$X_T\sim\pi_{\mathrm{old}}(\cdot|o)$,
and $X_0\perp X_T$.
We assume the pretrained flow matching model
learns the marginal velocity field of
this interpolation.
We define the tilted flow path by terminally reweighting the base joint law:
$p_{t,T}^\beta(x,a|o)
:=
    {
        r_\beta(o,a)
        p_{t,T}^\phi(x,a|o)
    }/{
        Z_{\beta}(o)
    },
    ~~ \text{where} ~~
    Z_{\beta}(o):=
    \E_{X_T\sim p_T^\phi(\cdot|o)}
    [
        r_\beta(o,X_T)
    ]$.
This gives the tilted terminal policy
$\pi_\beta(a|o)=p_T^\beta(a|o)\propto r_\beta(o,a)\pi_{\mathrm{old}}(a|o)$.
Let $u^\beta(x,t,o)$ denote the velocity inducing the marginals $p_t^\beta(\cdot|o)$ and terminal $\pi_\beta(\cdot|o)$.
We define the drift matching objective for flow policies as
\begin{align}
\label{eq:def_DM_fm}
    \mathcal{L}_{\mathrm{DM}}(\theta)
    =
    \E_{t,o, X_t \sim p_t^\beta(\cdot|o)}
    \Big[
        \big\|
            u^\beta(X_t,t,o)
            - u_\theta(X_t,t,o)
        \big\|_2^2
    \Big].
\end{align}
This objective is not implementable since it requires samples from $p_t^\beta$ and access to $u^\beta$.
Therefore, we rewrite \eqref{eq:def_DM_fm} as a conditional drift matching objective that only relies on the behavior process,
\begin{align}
\label{eq:def_cDM_fm}
    \mathcal{L}_{\mathrm{CDM}}(\theta)
    = & ~
    \E_{t,o, X_t \sim p_t^\phi, X_T \sim p_{T| t}^\phi}
    \Big[
        r_\beta(o,X_T)
        \big\|
            u^\beta(X_t| X_T,t,o)
            - u_\theta(X_t,t,o)
        \big\|_2^2
    \Big] \\
    \approx &~
    \E_{t,o, X_0 \sim p_0^\phi, X_T \sim p_T^\phi, X_t \sim p_{t|T,0}^\phi}
    \Big[
        r_\beta(o,X_T)
        \big\|
            \dot{\alpha}_t X_0 + \dot{\beta}_t X_T
            - u_\theta(X_t,t,o)
        \big\|_2^2
    \Big].
    \nonumber
\end{align}
The approximation is reliable when $u_\phi$ is well-trained under the affine path $X_t=\alpha_t X_0+\beta_t X_T$, independent coupling and the trust-region update keeps $\pi_{\mathrm{old}}$ close to the base behavior.
We summarize the training procedure in \cref{alg:rwdm_ppo_fm}.
Finally, we show $\mathcal{L}_{\mathrm{DM}}$ and $\mathcal{L}_{\mathrm{CDM}}$ are equivalent for learning.

\begin{theorem}[Gradient Equivalence for Flow]
\label{thm:grad_equi_fm}
Let $\mathcal{L}_{\mathrm{DM}}$ and $\mathcal{L}_{\mathrm{CDM}}$ be defined in \eqref{eq:def_DM_fm} and \eqref{eq:def_cDM_fm}, respectively.
Fix $o \in\mathcal S$, define normalizing constant
$Z_{\beta}(o) := \E_{X_T \sim p_T^\phi(\cdot|o)}\big[r_\beta(o,X_T)\big]$ and
\begin{align*}
    \ell_{\mathrm{DM}}(\theta,o)
    &:=
    \E_{t, X_t \sim p_t^\beta(\cdot|o)}
    \Big[
        \big\|
            u^\beta(X_t,t,o)
            - u_\theta(X_t,t,o)
        \big\|_2^2
    \Big].
    \\
    \ell_{\mathrm{CDM}}(\theta,o)
    &:=
    \E_{t, X_t \sim p_t^\phi, X_T \sim p_{T| t}^\phi}
    \Big[
        r_\beta(o,X_T)
        \big\|
            u^\beta(X_t| X_T,t,o)
            - u_\theta(X_t,t,o)
        \big\|_2^2
    \Big].
\end{align*}
Then their gradients satisfy
$\nabla_\theta \ell_{\mathrm{CDM}}(\theta,o)
    =
    Z_{\beta}(o) \cdot \nabla_\theta \ell_{\mathrm{DM}}(\theta,o)$.
\end{theorem}

\begin{proof}
Please see \cref{sec:proof_grad_equi_fm} for a detailed proof.
\end{proof}

\textbf{Ratio Regularization.}
In policy gradient methods, the probability ratio $r_\theta(o,a)=\pi_\theta(a|o)/\pi_{\mathrm{old}}(a|o)$ is induced by two normalized policies and therefore satisfies
$\E_{a\sim \pi_{\mathrm{old}}(\cdot|o)}[r_\theta(o,a)]=1$ for each state $o$.
In our setting, $r_\beta(o,a)$ is learned as an auxiliary model, so this normalization constraint is not automatically satisfied.
We therefore add a ratio regularizer that penalizes violations of $\E_{a\sim \pi_{\mathrm{old}}(\cdot|o)}[r_\beta(o,a)]=1$ in \cref{alg:rwdm_ppo,alg:rwdm_ppo_fm}, which stabilizes ratio learning.

\subsection{Unified Framework and Instantiations}
\label{sec:unified_framework}

DoobPO is a unified framework with two design choices.
The first is the generative policy class.
The second is the policy update method that specifies the target policy improvement.
Different choices along these two dimensions induce a broad family of diffusion/flow matching RL algorithms.

Under this view, several prior methods arise as special cases of DoobPO.
In particular, they correspond to policy updates that admit closed-form solutions.
In these cases, the ratio-learning stage becomes unnecessary, and DoobPO reduces to a single reweighted drift matching step.
For example, DPMD \cite{ma2025efficient} corresponds to a policy update with an analytic target ratio.
Under a KL-proximal policy mirror descent update \cite{zhan2023policy}, the updated policy satisfies
$\pi(a|o)\propto \pi_{\mathrm{old}}(a|o)\exp(\hat{A}_{\mathrm{old}}(o,a)/\lambda)$.
Therefore, the corresponding probability ratio satisfies
$r(o,a)\propto \exp(\hat{A}_{\mathrm{old}}(o,a)/\lambda)$.

For more general trust-region policy-gradient methods such as PPO \cite{schulman2017proximal} and GRPO \cite{shao2024deepseekmath}, the target ratio does not admit a closed form.
DoobPO then learns a lightweight ratio model and distills the resulting policy improvement into the generative policy through drift matching.
In this work, we derive DoobPPO and DoobGRPO for both diffusion policies and flow policies by instantiating the RL update with PPO and GRPO.
Hence, the same framework covers both prior closed-form methods and our new learned-ratio instantiations.
We summarize representative examples in \cref{tab:unified_framework}.

\subsection{Optimality from the Perspective of Stochastic Optimal Control}
\label{sec:soc}

\cref{sec:diff_dr_matching} shows that Doob's Drift Matching recovers the ratio-tilted action distribution through Doob's $h$-transform.
In this section, we establish a stronger optimality property from the perspective of endpoint-constrained stochastic optimal control (SOC).
Specifically, among all controlled reverse diffusions
with prescribed endpoint marginals $p_T^h$ and $p_0^h$,
the minimizer of conditional drift matching
induces the minimum-energy drift modification
of the pretrained reverse diffusion process.
Here, minimality is measured by the quadratic control energy.
We now formulate this result.

\begin{proposition}[Conditional Drift Matching Solves the SOC Problem]
\label{prop:cdm_solves_soc}
Let $h_t$ be defined in \eqref{eq:solution_exp}.
Fix a state $o \in\mathcal S$ and let $\mathbb P^\phi(\cdot|o)$ be the reference path measure induced by the reverse-time SDE
\begin{align*}
    \dd X_t
    =
    \bigl[
        f(X_t,t)-\sigma_t^2 s_\phi(X_t,t,o)
    \bigr]\dd t
    +
    \sigma_t\dd W_t,
\end{align*}
where $s_\phi(X_t,t,o)$ is the base model.
Consider the conditional drift matching objective
\begin{align*}
\mathcal{L}_{\mathrm{CDM}}(\theta)
=
\mathbb E_{t, X_t\sim p_t^\phi, X_0\sim p_{0|t}^\phi}
\left[
h_0(X_0,o)
\left\|
\nabla_x\log p_{0|t}^\phi(X_0|X_t,o)
-
s_\theta(X_t,t,o)
\right\|_2^2
\right].
\end{align*}
If the model class $\{s_\theta : \theta \in \Theta\}$ is realizable and the minimum of $L_{\mathrm{CDM}}$ is attained, then any minimizer $s_\theta^\star$ induces the optimal control
$u_t^\star(x,o) := -s_\theta^\star(x,t,o)$
that solves the SOC problem
\begin{align}
    & ~ \quad \quad \quad
    \quad \quad \quad \quad \quad \inf_u
    \frac12
    \mathbb E_{\mathbb Q^u(\cdot|o)}
    \int_0^T
    \|\sigma_t
    u_t(Y_t,o)\|_2^2\dd t,
    \label{eq:oj_soc}
    \\
    \text{s.t.} & ~
    \dd Y_t
    =
    \bigl[
        f(Y_t,t)-\sigma_t^2 s_\phi(Y_t,t,o) + \sigma_t^2 u_t(Y_t,o)
    \bigr]\dd t
    +
    \sigma_t\dd W_t ,
    ~ Y_T \sim p_T^h, ~Y_0 \sim p_0^h,
    \label{eq:sde_soc}
\end{align}
where $\mathbb Q^u(\cdot|o)$ is the path measure induced by the controlled reverse-time SDE in \eqref{eq:sde_soc}, initialized at $t=T$ with $Y_T\sim p_T^h$, and $p_0^h,p_T^h$ are the endpoint distributions given by \eqref{eq:p_t^h}.
\end{proposition}

\begin{proof}
Please see \cref{sec:proof_cdm_learn} for a detailed proof.
\end{proof}

\begin{remark}[Minimum-Energy Modification and Practical Implication]
\label{rem:ratio_h_interpretation}
Given a target terminal tilt, the SOC problem selects, among all controlled reverse diffusions that realize the tilted endpoint marginals, the one with minimum control energy.
The RL methods under our framework are special cases of this principle: we choose $h_0(a,o)\propto r_\beta(o,a)$, so that $p_0^h(a|o)\propto r_\beta(o,a)p_0^\phi(a|o)$.
\cref{prop:cdm_solves_soc} shows that Conditional Drift Matching does not only recover the reward-guided terminal policy; it learns the minimum-energy drift modification that realizes this policy.
In practice, under the drift parametrization $g_\theta$ and the practical loss \eqref{eq:cdm_loss_prac}, the population target satisfies $g_\theta^\star=s_\phi-u_t^\star$, where $u_t^\star(x,o)$ is the optimal SOC control in \cref{prop:cdm_solves_soc}.
Thus Conditional Drift Matching learns the closest feasible modification of the pretrained score model in the SOC sense.
This provides a theoretical explanation for the stability and efficiency of DoobPO.
\end{remark}

\section{Experiments}
\label{sec:exp}
\begin{figure}[!t]
    \centering
    \includegraphics[width=1\linewidth]{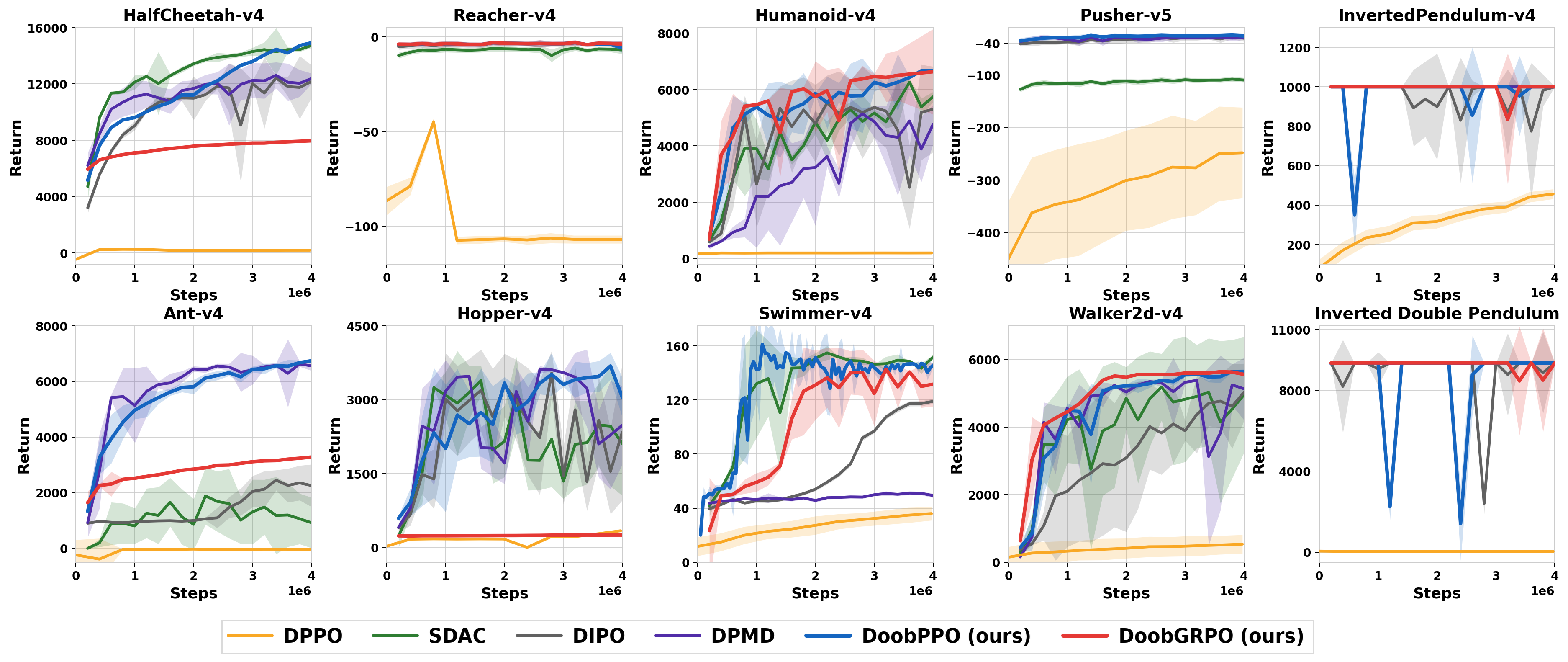}
    \caption{\small Learning curves on 10 MuJoCo continuous-control tasks, evaluated every 200k environment steps. Curves show mean undiscounted return over 3 training seeds, with shaded regions indicating one standard deviation. We compare DoobPO against strong diffusion-policy baselines DPPO, SDAC, DIPO, and DPMD.}
    \label{fig:mujoco10_dp}
\end{figure}

This section presents the experimental results.
We evaluate DoobPO on OpenAI Gym MuJoCo continuous-control benchmarks and on Robomimic manipulation tasks.
We describe the experimental setup in \cref{sec:exp_setup}.
We then report the results of RL fine-tuning diffusion and flow policies with DoobPO in \cref{sec:diffusion_exp} and \cref{sec:flow_exp}, and the Robomimic results in \cref{sec:robomimic_exp}.

\subsection{Experimental Setup}
\label{sec:exp_setup}

\textbf{Environments: OpenAI Gym MuJoCo.}
We evaluate our method on 10 OpenAI Gym MuJoCo continuous-control tasks covering locomotion, manipulation, and stabilization.
All diffusion-policy methods are trained from scratch for $4$M environment steps, while all flow-policy methods are trained from scratch for $5$M environment steps, under the same interaction budget within each setting.

\paragraph{Evaluation Metric.}
We use episode return as the evaluation metric, where higher return indicates better control performance.
During training, the policy is evaluated periodically, and at each evaluation point we report the mean return and its standard deviation over 3 random seeds.

\subsection{Diffusion Policy Comparison}
\label{sec:diffusion_exp}

\paragraph{Baselines.}
We compare DoobPO against recent diffusion-policy RL baselines, including DIPO \cite{yang2023policyrepresentationdiffusionprobability}, DPPO \cite{ren2024diffusion}, DPMD \cite{ma2025efficient}, DACER \cite{wang2024diffusionactorcriticentropyregulator}, and SDAC \cite{ma2025efficient}, as well as PPO \cite{schulman2017proximal}, TRPO \cite{schulman2015trust}, and SAC \cite{haarnoja2018soft}.
Please see \cref{sec:related_works} for details.

\begin{table*}[t]
\centering
\small
\setlength{\tabcolsep}{3.5pt}
\renewcommand{\arraystretch}{0.95}
\caption{\small Diffusion Policy Performance on MuJoCo environments
(mean $\pm$ std over 3 random seeds).
The best results are in bold and the second-best results are underlined.}
\label{tab:mujoco_grouped}
\begin{adjustbox}{max width=\textwidth}
\resizebox{0.9\linewidth}{!}{%
\begin{tabular}{llccccc}
\toprule
& Method
& HalfCheetah
& Reacher
& Humanoid
& Pusher
& InvertedPendulum \\
\midrule

\multirow{3}{*}{\textbf{Classic Model-Free RL}}
& PPO
& 1696 $\pm$ 223
& -3.63 $\pm$ 1.54
& 492 $\pm$ 113
& -25.28 $\pm$ 2.73
& \textbf{1000 $\pm$ 0} \\

& TRPO
& 2941 $\pm$ 417
& -3.83 $\pm$ 1.89
& 440 $\pm$ 82
& -31.84 $\pm$ 6.32
& \textbf{1000 $\pm$ 0} \\

& SAC
& 14422 $\pm$ 20
& -3.03 $\pm$ 1.21
& 5866 $\pm$ 16
& -23.96 $\pm$ 3.47
& \textbf{1000 $\pm$ 0} \\

\midrule

\multirow{7}{*}{\textbf{Diffusion Policy RL}}
& DIPO
& 12169 $\pm$ 1172
& -3.09 $\pm$ 1.36
& 5523 $\pm$ 980
& -27.85 $\pm$ 2.08
& \textbf{1000 $\pm$ 0} \\

& DPPO
& 248 $\pm$ 28
& -11.82 $\pm$ 1.28
& 195 $\pm$ 17
& -248.11 $\pm$ 86.32
& 456 $\pm$ 25 \\

& DPMD
& 12612 $\pm$ 932
& -2.97 $\pm$ 1.34
& 5138 $\pm$ 191
& -28.22 $\pm$ 6.06
& \textbf{1000 $\pm$ 0} \\

& DACER
& 14037 $\pm$ 370
& -2.973 $\pm$ 0.004
& 6460 $\pm$ 494
& -24.59 $\pm$ 0.37
& \textbf{1000 $\pm$ 0} \\

& SDAC
& \underline{14731 $\pm$ 166}
& -5.83 $\pm$ 1.36
& 6173 $\pm$ 1592
& -108 $\pm$ 5.23
& \textbf{1000 $\pm$ 0} \\

& \textbf{DoobPPO (ours)}
& \textbf{14921 $\pm$ 39}
& \underline{-2.96 $\pm$ 0.02}
& \textbf{6679 $\pm$ 63}
& \underline{-23.67 $\pm$ 3.16}
& \textbf{1000 $\pm$ 0} \\

& \textbf{DoobGRPO (ours)}
& 7800 $\pm$ 271
& \textbf{-2.95 $\pm$ 1.35}
& \underline{6634 $\pm$ 1517}
& \textbf{-23.62 $\pm$ 4.35}
& \textbf{1000 $\pm$ 0} \\

\midrule

& Method
& Ant
& Hopper
& Swimmer
& Walker2d
& Inverted2Pendulum \\
\midrule

\multirow{3}{*}{\textbf{Classic Model-Free RL}}
& PPO
& 1973 $\pm$ 67
& 1280 $\pm$ 197
& 121 $\pm$ 1.7
& 1242 $\pm$ 693
& \textbf{9360 $\pm$ 0} \\

& TRPO
& 3294 $\pm$ 76
& 3626 $\pm$ 9
& 42 $\pm$ 1.7
& 2765 $\pm$ 689
& 9188 $\pm$ 420 \\

& SAC
& 6534 $\pm$ 123
& 3533 $\pm$ 19
& 139 $\pm$ 1.1
& \textbf{6882 $\pm$ 60}
& \textbf{9360 $\pm$ 0} \\

\midrule

\multirow{7}{*}{\textbf{Diffusion Policy RL}}
& DIPO
& 2448 $\pm$ 361
& 3535 $\pm$ 456
& 119 $\pm$ 1.8
& 5028 $\pm$ 849
& 9351 $\pm$ 0.5 \\

& DPPO
& -29.9 $\pm$ 1.7
& 338 $\pm$ 5.1
& 36 $\pm$ 5
& 535 $\pm$ 278
& 32.7 $\pm$ 3.8 \\

& DPMD
& 6358 $\pm$ 276
& 3608 $\pm$ 73
& 51.1 $\pm$ 1.4
& 5381 $\pm$ 26
& 9355 $\pm$ 3.7 \\

& DACER
& \textbf{7137 $\pm$ 137}
& \textbf{3774 $\pm$ 106}
& 145.1 $\pm$ 5.4
& 5384 $\pm$ 1174
& 4096 $\pm$ 4778 \\

& SDAC
& 1880 $\pm$ 1025
& 3384 $\pm$ 303
& \underline{154 $\pm$ 5.1}
& 5181 $\pm$ 986
& 9353 $\pm$ 0.3 \\

& \textbf{DoobPPO (ours)}
& \underline{6738 $\pm$ 53}
& \underline{3681 $\pm$ 269}
& \textbf{161 $\pm$ 0.7}
& \underline{5642 $\pm$ 13}
& \underline{9356 $\pm$ 3} \\

& \textbf{DoobGRPO (ours)}
& 3278 $\pm$ 40
& 251 $\pm$ 1.4
& 142 $\pm$ 2.4
& 5636 $\pm$ 10
& \textbf{9360 $\pm$ 0} \\

\bottomrule
\end{tabular}}
\end{adjustbox}
\end{table*}

\paragraph{Diffusion Policy Results.}
Table~\ref{tab:mujoco_grouped} and Figure~1 report the
diffusion-policy results.
Overall, DoobPO achieves the highest mean returns among
diffusion-policy RL methods on 8 out of 10 MuJoCo tasks,
including ties.
DoobPPO obtains the best results on HalfCheetah, Humanoid,
and Swimmer, while DoobGRPO performs best on Reacher
and Pusher.
Both variants achieve competitive performance on the
remaining tasks, while DACER obtains the highest returns
on Ant and Hopper.

Compared with DIPO, DoobPPO improves performance by
22.6\% on HalfCheetah, 20.9\% on Humanoid,
175.2\% on Ant, and 12.2\% on Walker2d.
It also improves over DPMD by 18.3\%, 30.0\%,
6.0\%, and 4.9\% on the same tasks, respectively.
Against SDAC, DoobPPO achieves gains of 8.2\%
on Humanoid, 8.8\% on Hopper, and 8.9\%
on Walker2d.
The strong performance of DoobGRPO on Reacher and
Pusher further demonstrates that the framework is not
limited to a single PPO-style instantiation.
Together, these results establish DoobPO as a competitive
and flexible framework for diffusion-policy RL.

\subsection{Flow Policy Comparison}
\label{sec:flow_exp}

\begin{figure}[!t]
    \centering
    \includegraphics[width=1\linewidth]{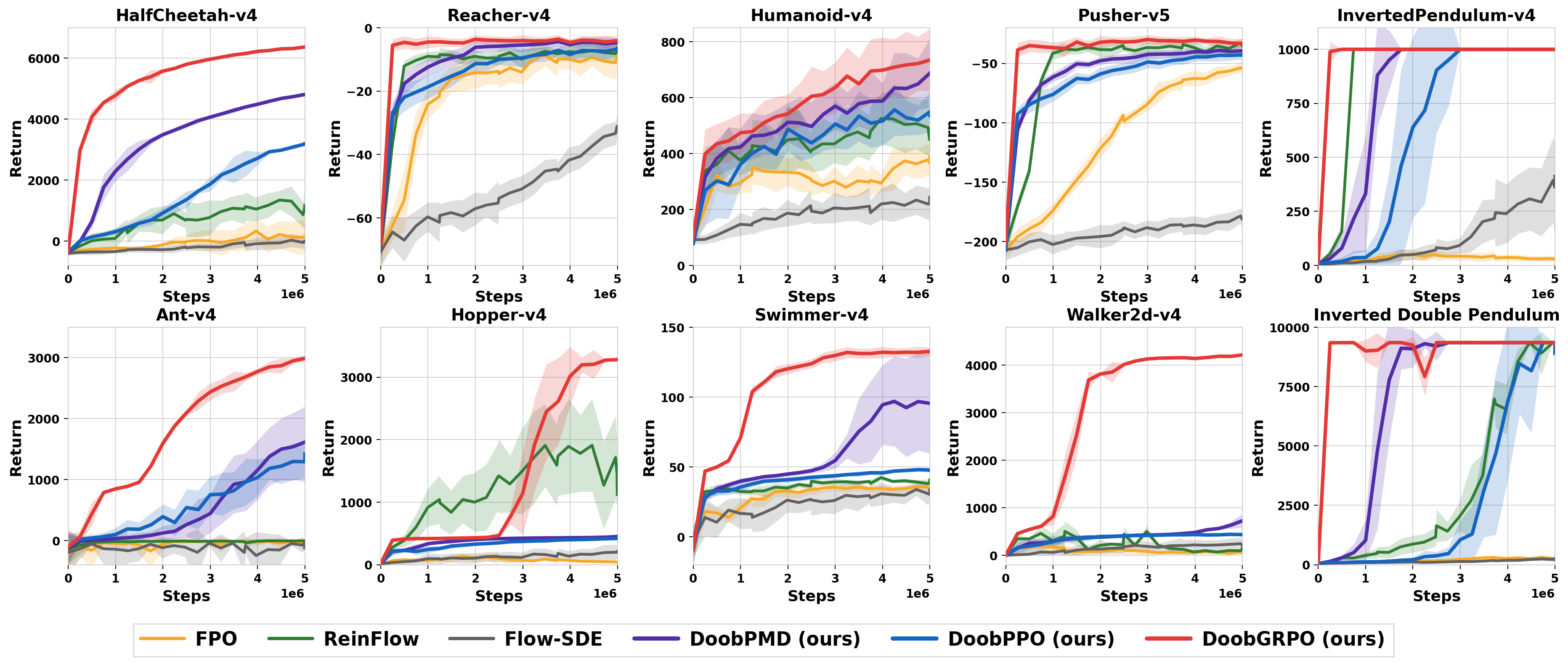}
    \caption{\small Learning curves on 10 MuJoCo continuous-control tasks, evaluated every 250k environment steps. Curves show mean undiscounted return over 3 training seeds, with shaded regions indicating one standard deviation. We compare DoobPO variants against flow-policy baselines FPO, ReinFlow, and Flow-SDE.}
    \label{fig:mujoco10_flow}
\end{figure}

\paragraph{Baselines.}
We compare DoobPO against three representative baselines for RL fine-tuning of flow policies: FPO \cite{mcallister2025flow}, ReinFlow \cite{zhang2025reinflow}, and Flow-SDE \cite{chen2025pirl}.
Please see \cref{sec:related_works} for more details.

\begin{table*}[t]
\centering
\small
\setlength{\tabcolsep}{3.5pt}
\renewcommand{\arraystretch}{0.95}
\caption{\small Flow Policy Performance on MuJoCo environments (mean $\pm$ std over 3 random seeds).}
\label{tab:mujoco_flow_grouped}
\begin{adjustbox}{max width=\textwidth}
\begin{tabular}{llccccc}
\toprule
& Method
& HalfCheetah
& Reacher
& Humanoid
& Pusher
& InvertedPendulum \\
\midrule

\multirow{6}{*}{\textbf{Flow Policy RL}}
& FPO
& 327 $\pm$ 399
& -8.79 $\pm$ 3.20
& 390 $\pm$ 110
& -53.77 $\pm$ 3.69
& 52 $\pm$ 24 \\

& ReinFlow
& 1344 $\pm$ 310
& -6.30 $\pm$ 2.83
& 525 $\pm$ 92
& \underline{-32.85 $\pm$ 3.81}
& \textbf{1000 $\pm$ 0} \\

& Flow-SDE
& 27 $\pm$ 219
& -31.08 $\pm$ 3.53
& 244 $\pm$ 60
& -178.42 $\pm$ 6.41
& 413 $\pm$ 238 \\

& \textbf{DoobPMD (ours)}
& \underline{4810 $\pm$ 53}
& \underline{-4.35 $\pm$ 1.40}
& \underline{688 $\pm$ 123}
& -39.10 $\pm$ 4.08
& \textbf{1000 $\pm$ 0} \\

& \textbf{DoobPPO (ours)}
& 3190 $\pm$ 60
& -6.73 $\pm$ 2.16
& 556 $\pm$ 94
& -42.81 $\pm$ 3.75
& \textbf{1000 $\pm$ 0} \\

& \textbf{DoobGRPO (ours)}
& \textbf{6368 $\pm$ 51}
& \textbf{-3.60 $\pm$ 1.37}
& \textbf{735 $\pm$ 111}
& \textbf{-29.90 $\pm$ 3.74}
& \textbf{1000 $\pm$ 0} \\

\midrule

& Method
& Ant
& Hopper
& Swimmer
& Walker2d
& Inverted2Pendulum \\
\midrule

\multirow{6}{*}{\textbf{Flow Policy RL}}
& FPO
& -13 $\pm$ 13
& 137 $\pm$ 55
& 35.8 $\pm$ 2.51
& 186 $\pm$ 140
& 314 $\pm$ 84 \\

& ReinFlow
& -5 $\pm$ 12
& \underline{1911 $\pm$ 740}
& 42.3 $\pm$ 1.77
& 503 $\pm$ 213
& \textbf{9358 $\pm$ 1} \\

& Flow-SDE
& -45 $\pm$ 41
& 210 $\pm$ 100
& 33.9 $\pm$ 4.44
& 243 $\pm$ 128
& 268 $\pm$ 68 \\

& \textbf{DoobPMD (ours)}
& \underline{1614 $\pm$ 583}
& 452 $\pm$ 4
& \underline{96.9 $\pm$ 32.13}
& \underline{726 $\pm$ 144}
& 9348 $\pm$ 0 \\

& \textbf{DoobPPO (ours)}
& 1417 $\pm$ 203
& 423 $\pm$ 10
& 48.0 $\pm$ 1.30
& 440 $\pm$ 16
& \underline{9350 $\pm$ 1} \\

& \textbf{DoobGRPO (ours)}
& \textbf{2986 $\pm$ 56}
& \textbf{3282 $\pm$ 7}
& \textbf{132.7 $\pm$ 2.71}
& \textbf{4215 $\pm$ 13}
& \underline{9350 $\pm$ 1} \\

\bottomrule
\end{tabular}
\end{adjustbox}
\end{table*}

\paragraph{Flow Policy Results.}
Table~\ref{tab:mujoco_flow_grouped} and Figure~\ref{fig:mujoco10_flow} report the flow-policy results.
Overall, DoobGRPO achieves the strongest performance across the MuJoCo benchmark under the same 5M-step training budget, obtaining the best results on 9 out of 10 tasks (including ties).
Compared with prior flow-policy RL baselines, DoobGRPO gives large gains on key locomotion tasks.
For example, it improves over the strongest baseline ReinFlow by \textbf{373.8\%} on \texttt{HalfCheetah}, \textbf{71.7\%} on \texttt{Hopper}, \textbf{213.7\%} on \texttt{Swimmer}, and \textbf{738.0\%} on \texttt{Walker2d}.
On the manipulation tasks, DoobGRPO also achieves the best performance on \texttt{Reacher} and \texttt{Pusher}, improving over ReinFlow by \textbf{42.9\%} and \textbf{9.0\%}.
DoobPMD serves as a strong second-best on several tasks, including \texttt{HalfCheetah}, \texttt{Reacher}, and \texttt{Humanoid}.
This further demonstrates the generality of the DoobPO framework across different policy-gradient methods.
These results show that DoobPO provides a useful RL framework for flow policies.

\subsection{Robomimic Manipulation}
\label{sec:robomimic_exp}

\paragraph{Setup.}
To evaluate DoobPO beyond MuJoCo continuous control, we test it on the \texttt{Can} and \texttt{Square} manipulation tasks from Robomimic \cite{mandlekar2021matters}.
Both tasks use low-dimensional state observations, sparse binary success rewards, and an action-chunk horizon of 4.
The maximum episode lengths are 300 primitive steps for \texttt{Can} and 400 for \texttt{Square}.
We follow the official Robomimic evaluation protocol with 50 parallel environments.
An episode is successful if the binary task-completion reward is achieved at least once, and the success rate is computed over all completed evaluation episodes.
We compare DoobPPO with DPPO \cite{ren2024diffusion}.
Both methods are initialized from the same pretrained diffusion behavior-cloning policy and trained for 151 iterations on \texttt{Can} and 200 iterations on \texttt{Square}, and results are averaged over 3 random seeds.
The DoobPPO hyperparameters are listed in \cref{tab:hyperparameters_robomimic}.

\paragraph{Results.}
DoobPPO matches or slightly improves upon DPPO on both tasks, reaching success rates of 99.46\% versus 99.23\% on \texttt{Can} and 99.78\% versus 99.59\% on \texttt{Square}.
Since both tasks are close to saturation, the main takeaway is not a large performance gap but that DoobPO extends to sparse-reward manipulation settings.

\subsection{Implementation Details and Ablation Studies}

Additional implementation details are provided in \cref{sec:implementation}, including computational resources, training hyperparameters, and other reproducibility details.
We also include ablation studies in \cref{sec:ablation} to examine the effect of the ratio network size and the number of sampling steps.

\section{Discussion and Conclusion}
\label{sec:conclusion}In this work, we introduce \textbf{Doob Policy Optimization (DoobPO)}, a unified framework for RL fine-tuning of diffusion and flow policies.
The framework supports a broad family of policy updates and avoids exact action-likelihood evaluation by separating policy improvement from generative policy learning.
Theoretically, we justify DoobPO from two directions.
First, we characterize the ratio-tilted diffusion process through Doob's $h$-transform
(\cref{lem:algo_correctness}).
We then show that, for each fixed state, the gradient of the conditional drift matching objective is proportional to that of the ideal objective for both diffusion and flow policies
(\cref{thm:grad_equi_diff,thm:grad_equi_fm}).
Second, under the stated assumptions, the population minimizer of conditional drift matching induces a diffusion process that solves the corresponding endpoint-constrained stochastic optimal control problem
(\cref{prop:cdm_solves_soc}).
The resulting process is the minimum-energy modification of the reference diffusion that realizes the prescribed tilted endpoint marginals.
This provides a control-theoretic interpretation of how the desired policy update is transferred into generative dynamics.
Experimentally, we evaluate DoobPO on MuJoCo benchmarks with both diffusion and flow policies and observe superior overall performance against strong baselines
(\cref{sec:diffusion_exp,sec:flow_exp}).

More broadly, DoobPO provides a common interface between policy optimization and generative model training.
Policy updates with analytic ratios and updates represented by learned ratio models are accommodated within the same conditional drift matching framework
(\cref{sec:unified_framework}).
This separates the choice of the policy update rule from the choice of the generative policy class.
Nevertheless, our study mainly focuses on continuous generation, and the empirical evaluation is limited to MuJoCo continuous-control benchmarks and Robomimic manipulation tasks.
The theoretical results characterize population objectives under the stated assumptions.
An important direction for future work is to quantify how ratio estimation error and approximate drift matching affect the realized policy update.
We also plan to extend the framework to discrete generation and evaluate its applicability to more complex policy learning settings.

\appendix
\label{sec:append}
\part*{Appendix}
{
\setlength{\parskip}{-0em}
\startcontents[sections]
\printcontents[sections]{ }{1}{}
}

{
\setlength{\parskip}{-0em}
\startcontents[sections]
\printcontents[sections]{ }{1}{}
}
\clearpage

\section*{Impact Statement}
This work advances the methodological and theoretical understanding of diffusion and flow policy optimization and presents no foreseeable negative social impacts.
\section*{LLM Usage Disclosure}
We used large language models (LLMs) to aid and polish writing, such as improving clarity, grammar, and conciseness. We also used LLMs for retrieval and discovery, for example exhausting literature to identify potential missing related work. All technical content, proofs, experiments, and results are original contributions by the authors.

\section{Notation}
\label{sec:notation}
We use $\E[\cdot]$ and $\P(\cdot)$ to denote expectation and probability.
Unless otherwise specified, any expectation involving a random time $t$ is taken with respect to the uniform distribution on $[0,T]$, i.e., $t\sim \mathrm{Unif}([0,T])$.
For a time horizon $T>0$, let $\Omega=C([0,T];\mathbb R^d)$ be the trajectory space.
For any path measure $\mathbb Q$ on $\Omega$, we denote its time-$t$ marginal by $q_t$ and its conditional distribution of $X_t$ given $X_r=x_r$ by $q_{t|r}(\cdot|x_r)$.
For diffusion policies, $t=0$ denotes the action endpoint and $t=T$ denotes the noise endpoint, and the reverse denoising process is simulated from $T$ to $0$.
For flow policies, the convention is reversed: $t=0$ denotes the noise endpoint and $t=T$ denotes the action endpoint.
We write $p_t^\phi(\cdot|o)$ for the marginal distributions induced by the pretrained score model $s_\phi$ or velocity model $u_\phi$, and $p_t^h(\cdot|o)$ for the marginal distributions induced by the Doob-tilted process introduced in \cref{sec:method}.
The learned ratio model is denoted by $r_\beta(o,a)$, and in the RL setting the terminal tilt is chosen proportional to this learned ratio.
In \cref{sec:soc}, $\mathbb Q^u(\cdot|o)$ denotes the controlled path measure induced by the control $u_t(\cdot,o)$, and $q_t^u(\cdot|o)$ denotes its time-$t$ marginal distribution.
When no ambiguity arises, we use the same lowercase notation for a probability distribution and its density.
Unless otherwise specified, the state distribution used in expectations is assumed to have full support on $\mathcal S$.
We write $\mathbb R_+ := (0,\infty)$.
We say the model class $\{s_\theta:\theta\in\Theta\}$ is realizable for the Doob drift if there exists $\theta^\star\in\Theta$ such that
$s_{\theta^\star}(x,t,o)=\nabla_x\log h_t(x,o)$
for every $x\in\mathcal A$, $t\in[0,T]$, and $o \in\mathcal S$.

\section{Related Works}
\label{sec:related_works}

\paragraph{Diffusion Model and Diffusion Policy.}
Diffusion models \cite{sohldickstein2015deep,ho2020denoising} are probabilistic generative models that progressively transform noise into samples from a target distribution. They are closely connected to score-based generative modeling \cite{song2019generative,song2021scorebased}, which offers a complementary view based on learning the gradient field of the data distribution and unifies diffusion-based generation under the stochastic differential equation framework. Owing to their strong expressiveness for complex, high-dimensional, and multimodal distributions, diffusion models have become increasingly attractive for control and decision-making. In particular, diffusion policies parameterize action distributions with diffusion models and generate actions through iterative denoising, yielding richer policy classes than standard unimodal Gaussian actors. Prior work has shown strong empirical performance of diffusion policies in visuomotor control and robotic manipulation \cite{chi2025diffusion}, as well as in goal-conditioned imitation learning and language- or play-conditioned skill learning \cite{reuss2023goalconditionedimitationlearningusing,chen2023playfusionskillacquisitiondiffusion}.

More recently, diffusion policies have been extended to online reinforcement learning. Methods such as DPPO and related diffusion-policy RL approaches show that diffusion-based policy classes can be optimized effectively beyond pure imitation learning \cite{ren2024diffusion,ma2025efficient}. Our work builds on this line of research, but focuses on RL fine-tuning of diffusion policies in the likelihood-intractable setting.

We compare against five recent diffusion-policy RL baselines: DIPO \cite{yang2023policyrepresentationdiffusionprobability}, DPPO \cite{ren2024diffusion}, DPMD \cite{ma2025efficient}, DACER \cite{wang2024diffusionactorcriticentropyregulator}, and SDAC \cite{ma2025efficient}. These baselines cover both off-policy and on-policy diffusion RL methods: DIPO, DPMD, DACER, and SDAC are off-policy, whereas DPPO is on-policy. DIPO adopts a two-stage strategy that first updates state-action particles using value gradients and then distills them into a diffusion policy. 
DPPO formulates policy optimization over a diffusion-augmented MDP and applies Proximal Policy Optimization to the resulting process. 
DPMD studies efficient online RL for diffusion policies through a policy-mirror-descent-style update with an analytic target ratio. DACER directly backpropagates value gradients through the reverse diffusion process and introduces entropy regularization to improve exploration. SDAC extends diffusion-policy optimization to the soft-actor-critic regime with an analytic-ratio-based update. 

Our work generalizes the analytic-ratio ideas underlying DPMD and SDAC into a unified likelihood-free framework, extending them from diffusion policies to flow matching policies and to a broader class of policy optimization methods, including PPO \cite{schulman2017proximal} and GRPO \cite{shao2024deepseekmath}.

\paragraph{Flow Matching Model and Flow Policy.}
Flow matching \cite{lipman2022flow,liu2022flow} provides an alternative to diffusion models by learning a deterministic ODE transport between a noise prior and the data distribution. Compared to diffusion models, flow matching avoids simulating stochastic dynamics during training and offers faster sampling through fewer integration steps. These properties have made flow matching models increasingly popular in decision-making applications. $\pi_0$ \cite{black2024pi} and its successor $\pi_{0.5}^*$ \cite{intelligence2025pi} adopt flow matching as the generative backbone of vision-language-action models for general robot manipulation, demonstrating that flow policies can scale to diverse real-world tasks. However, the likelihood of a flow policy is as intractable as that of a diffusion policy, since it requires solving an ODE and computing the associated log-determinant of the Jacobian. This intractability poses the same challenge for online RL fine-tuning: standard policy gradient methods that rely on likelihood ratios cannot be directly applied. Flow Q-Learning \cite{park2025flow} is an early attempt that sidesteps this issue by using Q-function guidance to steer the flow sampling process, analogous to Q-guided diffusion methods. For policy gradient approaches, several workarounds have been proposed. FPO \cite{mcallister2025flow} is a PPO-style method that replaces explicit policy ratios with an implicit surrogate derived from the flow matching training loss, avoiding likelihood computation entirely. ReinFlow \cite{zhang2025reinflow} converts a deterministic flow policy into a stochastic short-horizon Markov process via learned noise injection, enabling tractable stepwise likelihood computation for policy gradient updates. Flow-SDE \cite{chen2025pirl} instead replaces the deterministic flow ODE with an analytically constructed stochastic differential equation, yielding explicit transition probabilities suitable for policy gradient estimation. All these methods require some form of likelihood evaluation or surrogate at the denoising steps; in contrast, our approach bypasses this requirement entirely by learning a lightweight ratio network that directly estimates the probability ratio needed for policy optimization.

\section{\texorpdfstring{Proof of \cref{lem:algo_correctness}}{Proof of Theorem \ref{lem:algo_correctness}}}
\label{sec:proof_corr}
\begin{lemma}[\cref{lem:algo_correctness} Restated: Doob's $h$-Transform]
Fix a state $o \in\mathcal{S}$.
Let $(X_t)_{t\in[0,T]}$ follow the reverse-time denoising SDE \eqref{eq:reverse_sde} with $\nabla_x \log p_t(x|o)$ replaced by its estimator $s_\phi(x,t,o)$, inducing conditional marginals 
$\{p_t^\phi(\cdot|o)\}_{t\in[0,T]}$.
Let $h_0(\cdot,o):\mathcal{A}\to\R_+$ be a positive integrable function.
Let $h(\cdot,o,\cdot):\mathcal{A}\times[0,T]\to \R_+$ be the associated Doob $h$-function, denoted by $h_t(x,o)$, defined as
\begin{align*}
    h_t(x,o)
    =
    \E \big[h_0(X_0,o)| X_t=x,o\big],
\end{align*}
where the conditional expectation is taken under the original (untilted) denoising process induced by $s_\phi$.
The Doob $h$-transformed denoising process starts from
\[
    X_T^h \sim p_T^h(\cdot|o),
    \qquad
    p_T^h(x|o)
    :=
    \frac{
        p_T^\phi(x|o) h_T(x,o)
    }{
        \E_{X_0\sim p_0^{\phi}(\cdot|o)}[h_0(X_0,o)]
    },
\]
and evolves backward from $T$ to $0$ according to
\begin{align*}
    \dd X_t^{h}
    =
    \Big[
        f(X_t^{h},t)
        -\sigma_t^2 s_\phi(X_t^{h},t,o)
        -\sigma_t^2 \nabla_x \log h_t(X_t^{h},o)
    \Big]\dd t
    +\sigma_t\dd W_t.
\end{align*}
Then its conditional marginals satisfy
\begin{align*}
    p_t^{h}(x|o)
    =
    \frac{
        p_t^\phi(x|o)h_t(x,o)
    }{
        \E_{X_0\sim p_0^{\phi}(\cdot|o)}[h_0(X_0,o)]
    },
    \quad t\in[0,T].
\end{align*}
In particular, taking $h_0(x,o)=r_\beta(o,x)$ yields
\[
    p_0^h(x|o)
    \propto
    p_0^{\phi}(x|o)r_\beta(o,x).
\]
Since $p_0^{\phi}(\cdot|o)=\pi_{\mathrm{old}}(\cdot|o)$, the time-$0$ policy induced by the tilted denoising process satisfies
\begin{align*}
    \pi(a|o)
    \propto
    r_\beta(o,a)\pi_{\mathrm{old}}(a|o).
\end{align*}
\end{lemma}

\begin{proof}
Conditioning on the fixed state $o$, the original denoising process is a reverse-time Markov process
with marginals $\{p_t^\phi(\cdot|o)\}_{t\in[0,T]}$ and time-$0$ tilting function $h_0(\cdot,o)$.
By the continuous-time Doob $h$-transform for reverse-time diffusion processes
\cite[Proposition~2]{nguyen2025h}, the transformed process evolves backward from $T$ to $0$ according to
\begin{align*}
    \dd X_t^{h}
    =
    \Big[
        f(X_t^{h},t)
        -\sigma_t^2 s_\phi(X_t^{h},t,o)
        -\sigma_t^2 \nabla_x \log h_t(X_t^{h},o)
    \Big]\dd t
    +\sigma_t\dd W_t.
\end{align*}
which is exactly \eqref{eq:new_process}.
It remains to identify its marginals. By the reverse-time Doob bridge marginal identity
\cite[Proposition~1]{nguyen2025h}, also see \cite[Section~3.2]{zhu2026training}, if the transformed process starts from
\[
    p_T^h(x|o)
    =
    \frac{
        p_T^\phi(x|o)h_T(x,o)
    }{
        \E_{X_0\sim p_0^\phi(\cdot|o)}[h_0(X_0,o)]
    },
\]
then, for every $t\in[0,T]$,
\[
    p_t^h(x|o)
    =
    \frac{
        p_t^\phi(x|o)h_t(x,o)
    }{
        \E_{X_0\sim p_0^\phi(\cdot|o)}[h_0(X_0,o)]
    }.
\]
This proves \eqref{eq:p_t^h}.
Taking $t=0$ gives
\[
    p_0^h(a|o)
    =
    \frac{
        p_0^\phi(a|o)h_0(a,o)
    }{
        \E_{X_0\sim p_0^\phi(\cdot|o)}[h_0(X_0,o)]
    }.
\]
In particular, with $h_0(a,o)=r_\beta(o,a)$,
\[
    p_0^h(a|o)
    \propto
    p_0^\phi(a|o)r_\beta(o,a).
\]
Since $p_0^\phi(\cdot|o)=\pi_{\mathrm{old}}(\cdot|o)$, the time-$0$ policy induced by the tilted
denoising process satisfies
\[
    \pi(a|o)
    \propto
    r_\beta(o,a)\pi_{\mathrm{old}}(a|o).
\]
\end{proof}

\section{\texorpdfstring{Proof of \cref{thm:grad_equi_diff}}{Proof of Theorem \ref{thm:grad_equi_diff}}}
\label{sec:proof_grad_equi}

\begin{theorem}[\cref{thm:grad_equi_diff} Restated: Gradient Equivalence for Diffusion]
\label{thm:grad_equi_re}
    Let $\mathcal{L}_{\mathrm{DM}}$ and $\mathcal{L}_{\mathrm{CDM}}$ be defined in \eqref{eq:def_DM} and \eqref{eq:def_cDM}, respectively.
Fix state $o \in\mathcal S$, define normalizing constant
$Z_{\beta}(o) := \E_{X_0 \sim p_0^\phi(\cdot|o)}\big[r_\beta(o,X_0)\big]$and 
\begin{align*}
    \ell_{\mathrm{DM}}(\theta,o)
    &:=
    \E_{t,X_t\sim p_t^h(\cdot|o)}
    \left[
        \|
        \nabla_x\log h_t(X_t,o)-s_\theta(X_t,t,o)
        \|_2^2
    \right],
    \\
    \ell_{\mathrm{CDM}}(\theta,o)
    &:=
    \E_{t,X_t\sim p_t^\phi(\cdot|o),X_0\sim p_{0|t}^\phi(\cdot|X_t,o)}
    \left[
        r_\beta(o,X_0)
        \|
        \nabla_x\log p_{0|t}^\phi(X_0|X_t,o)
        -
        s_\theta(X_t,t,o)
        \|_2^2
    \right].
\end{align*}
Then their gradients satisfy
\begin{align*}
    \nabla_\theta \ell_{\mathrm{CDM}}(\theta,o)
    =
    Z_{\beta}(o) \cdot \nabla_\theta \ell_{\mathrm{DM}}(\theta,o).
\end{align*}
\end{theorem}

\begin{proof}
For notational simplicity, we first present the proof in the unconditional setting.
That is, the conditioning variable $o$ is fixed and omitted from the notation.
The conditional statement follows by applying the same argument pointwise for each $o \in\mathcal S$, with
$p_t(\cdot)$, $p_t^h(\cdot)$, $p_{0|t}(\cdot|\cdot)$, $h_t(\cdot)$, and $r_\beta(\cdot)$ replaced by
$p_t^\phi(\cdot|o)$, $p_t^h(\cdot|o)$, $p_{0|t}^\phi(\cdot|\cdot,o)$, $h_t(\cdot,o)$, and $r_\beta(o,\cdot)$, respectively.
We now expand the unconditional drift matching loss.
\begin{align}
\label{eq:derivation_all_sm}
    \ell_{\text{DM}}(\theta)
    = & ~ \E_{t, X_t \sim p_t^h}
    [\| \nabla \log h_t(X_t)
    - s_\theta(X_t,t)\|_2^2] \\
    =\frac1T & ~ \int_0^T \int_{\R^d}
    \left[\|\nabla_x \log h_t(x)\|_2^2 + \|s_\theta(x,t)\|_2^2 - 2s_\theta(x,t)^\top \nabla_x \log h_t(x)\right]
    p_t^h(x)
    \dd x \dd t \nonumber\\
    =\frac1T & ~ \int_0^T \int_{\R^d} \left[\|s_\theta
    (x,t)\|_2^2 - 2
    s_\theta(x,t)^\top \nabla_x \log h_t(x)\right]
    p_t^h(x) 
    \dd x \dd t + C_1,
\end{align}
where $C_1 := \frac1T \int_0^T \int_{\R^d}
\|\nabla \log h_t(x)\|_2^2 
p_t^h(x) \dd x \dd t~$ is independent to $\theta$.
For the first term, 
\begin{align}
    & ~\frac1T  \int_0^T \int_{\R^d} \|s_\theta
    (x,t)\|_2^2 
    p_t^h(x)  
    \dd x \dd t \nonumber\\
    = & ~ \frac{1}{Z_{\beta}(o)T} \int_0^T \int_{\R^d} \|s_\theta
    (x,t)\|_2^2 \cdot
    h_t(x) p_t(x) 
    \dd x \dd t \annot{By $p_t^h(x) \propto p_t(x) h_t(x)$} \\
    = & ~ \frac{1}{Z_{\beta}(o)} \E_{t,X_t \sim p_t}[\|s_\theta
    (X_t,t)\|_2^2 h(X_t,t)] \nonumber\\
    = & ~ \frac{1}{Z_{\beta}(o)} \E_{t,X_t \sim p_t}[\|s_\theta
    (X_t,t)\|_2^2 \E_{X_0 \sim p_{0|t}}[r_\beta(X_0)|X_t]] \annot{By $ h(X_t,t)=
    \E[h(X_0,0)|X_t]$ follows from \eqref{eq:solution_exp}}\\
    = & ~ \frac{1}{Z_{\beta}(o)} \E_{t,X_t \sim p_t, X_0 \sim p_{0|t}}[\|s_\theta
    (X_t,t)\|_2^2 \cdot r_\beta(X_0)],
    \label{eq:derivation_all_first}
\end{align}
where $Z_{\beta}(o) = \int_{\R^d} p_t(x) h_t(x) \dd x = \E_{X_t \sim p_t}[h(X_t,t)] = \E_{X_0 \sim p_0}[r_\beta(X_0)]$ is the normalizing constant and the last equation follows from \eqref{eq:solution_exp} and tower property.  
Next we focus on the second term.
We show that
\begin{align*}
    & ~ \frac1T \int_0^T \int_{\R^d}
    s_\theta(x,t)^\top \nabla_x \log h_t(x)
    p_t^h(x)
    \dd x \dd t \\
    = & ~ \frac1T \int_0^T \int_{\R^d}
    s_\theta(x,t)^\top \frac{\nabla_x h_t(x)}{h_t(x)}
    p_t^h(x) 
    \dd x \dd t \annot{By $\nabla_x \log h_t(x) = \nabla_x h_t(x)/ h_t(x)$}\\
    = & ~ \frac{1}{Z_{\beta}(o)T=}\int_0^T \int_{\R^d}
    s_\theta(x,t)^\top \frac{\nabla_x h_t(x)}{h_t(x)}
    p_t(x) h_t(x) 
    \dd x \dd t \annot{By $p_t^h(x) \propto p_t(x) h_t(x)$}\\
    = & ~ \frac{1}{Z_{\beta}(o)T}\int_0^T \int_{\R^d}
    s_\theta(x,t)^\top 
    \nabla_x h_t(x)
    p_t(x)  
    \dd x \dd t.
\end{align*}
Next, we consider conditioning by recalling $h_t(x) = \int_{\R^d} r_\beta(z) p_{0|t}(z|x) \dd z$ follows from \eqref{eq:solution_exp}.
\begin{align}
\label{eq:derivation_second_t}
    & ~ \frac1T\int_0^T \int_{\R^d}
    s_\theta(x,t)^\top \nabla_x  \log h_t(x)
    p_t^h(x)
    \dd x \dd t \nonumber\\
    = & ~ \frac{1}{Z_{\beta}(o)T}\int_0^T \int_{\R^d}
    s_\theta(x,t)^\top 
    \left(\int_{\R^d} r_\beta(z) \nabla_{x} p_{0|t}(z|x) \dd z \right)
    p_t(x)  
    \dd x \dd t \nonumber\\
    = & ~ \frac{1}{Z_{\beta}(o)T}\int_0^T \int_{\R^d}
    s_\theta(x,t)^\top 
    \left(\int_{\R^d} r_\beta(z) 
    \frac{\nabla_{x} p_{0|t}(z|x)}{p_{0|t}(z|x)} \cdot p_{0|t}(z|x) \dd z \right)
    p_t(x)  
    \dd x \dd t \nonumber\\
    = & ~ \frac{1}{Z_{\beta}(o)T}\int_0^T \int_{\R^d}
    s_\theta(x,t)^\top 
    \int_{\R^d} r_\beta(z) 
    \nabla_x \log p_{0|t}(z|x) \cdot 
    p_{0|t}(z|x) p_t(x)  
    \dd z \dd x \dd t \nonumber \\
    = & ~ \frac{1}{Z_{\beta}(o)}
    \E_{t, X_t \sim p_t, X_0 \sim p_{0|t}}[r_\beta(X_0) \cdot s_\theta(X_t,t)^\top \nabla_{X_t} \log p_{0|t}(X_0|X_t)].
\end{align}
Substituting \eqref{eq:derivation_all_first} and \eqref{eq:derivation_second_t} back to \eqref{eq:derivation_all_sm}, we show that 
\begin{align*}
    \ell_{\text{DM}}(\theta)
    = \frac{1}{Z_{\beta}}
    \E_{t, X_t \sim p_t, X_0 \sim p_{0|t}}[ r_\beta(X_0) \cdot \left(\|s_\theta(X_t,t)\|_2^2 -2 s_\theta(X_t,t)^\top \nabla_{X_t} \log p_{0|t}(X_0|X_t) \right)] + C_1.
\end{align*}
By the definition of Conditional Drift Matching loss, 
we have
\begin{align*}
    & ~ \ell_{\text{CDM}}(\theta) \\
    =  & ~ \E_{t, X_t \sim p_t, X_0 \sim p_{0|t}}[r_\beta(X_0) \cdot \|\nabla_{X_t} \log p_{0|t}(X_0|X_t) - s_\theta(X_t,t) \|_2^2] \annot{By definition \eqref{eq:def_cDM}}\\
    = & ~ \E_{t, X_t \sim p_t, X_0 \sim p_{0|t}}[ r_\beta(X_0) \cdot \left(\|s_\theta
    (X_t,t)\|_2^2 -2 s_\theta(X_t,t)^\top \nabla_{X_t} \log p_{0|t}(X_0|X_t) \right)] + C_2,
\end{align*}
where $C_2 := \E_{t, X_t \sim p_t, X_0 \sim p_{0|t}}[ \| \nabla_{X_t} \log p_{0|t}(X_0|X_t)\|_2^2 \cdot r_\beta(X_0)]$ is independent to $\theta$.
Therefore, 
\begin{align*}
    \ell_{\text{CDM}}(\theta)
    = Z_{\beta}(o)(\ell_{\text{DM}}(\theta)-C_1)+C_2.
\end{align*}
Taking gradients with respect to $\theta$ gives
\begin{align*}
    \nabla_\theta  \ell_{\mathrm{CDM}}(\theta)
    =
    Z_{\beta}(o) \nabla_\theta \ell_{\mathrm{DM}}(\theta),
\end{align*}
since $C_1$ and $C_2$ are independent of $\theta$.

\end{proof}

\section{\texorpdfstring{Proof of \cref{thm:grad_equi_fm}}{Proof of Theorem \ref{thm:grad_equi_fm}}}
\label{sec:proof_grad_equi_fm}
\begin{theorem}[\cref{thm:grad_equi_fm} Restated: Gradient Equivalence for Flow]
Let $\mathcal{L}_{\mathrm{DM}}$ and $\mathcal{L}_{\mathrm{CDM}}$ be defined in \eqref{eq:def_DM_fm} and \eqref{eq:def_cDM_fm}, respectively.
Fix $o \in\mathcal S$, define normalizing constant
$Z_{\beta}(o) := \E_{X_T \sim p_T^\phi(\cdot|o)}\big[r_\beta(o,X_T)\big]$and 
\begin{align*}
    \ell_{\mathrm{DM}}(\theta,o)
    &:=
    \E_{t, X_t \sim p_t^\beta(\cdot|o)}
    \Big[
        \big\|
            u^\beta(X_t,t,o)
            - u_\theta(X_t,t,o)
        \big\|_2^2
    \Big].
    \\
    \ell_{\mathrm{CDM}}(\theta,o)
    &:=
    \E_{t, X_t \sim p_t^\phi, X_T \sim p_{T| t}^\phi}
    \Big[
        r_\beta(o,X_T)
        \big\|
            u^\beta(X_t| X_T,t,o)
            - u_\theta(X_t,t,o)
        \big\|_2^2
    \Big].
\end{align*}
Then their gradients satisfy
\begin{align*}
    \nabla_\theta \ell_{\mathrm{CDM}}(\theta,o)
    =
    Z_{\beta}(o) \cdot \nabla_\theta \ell_{\mathrm{DM}}(\theta,o).
\end{align*}
\end{theorem}
\begin{proof}
For each \(t\in[0,T]\), define $\bar r_t(x,o)$
\[
\bar r_t(x,o)
:=
\E_{X_T\sim p_{T|t}^\phi(\cdot|x,o)}
\big[
r_\beta(o,X_T)
\big].
\]
By the ratio-tilted construction of the flow path, the marginal
\(p_t^\beta(\cdot|o)\) satisfies
\begin{align}
Z_{\beta}(o)p_t^\beta(x|o)
=
p_t^\phi(x|o)\bar r_t(x,o).
\label{eq:fm-marginal-reweighting}
\end{align}
Consequently, for any fixed \(t\in[0,T]\) and any integrable test
function \(\psi\), we have
\begin{align}
\begin{aligned}
&Z_{\beta}(o)
\E_{X_t\sim p_t^\beta(\cdot|o)}
\big[
\psi(X_t)
\big]
\\
&=
Z_{\beta}(o)\int \psi(x)p_t^\beta(x|o)\dd x
\\
&=
\int \psi(x)p_t^\phi(x|o)\bar r_t(x,o)\dd x
\\
&=
\int \psi(x)p_t^\phi(x|o)
\left[
\int r_\beta(o,a)p_{T|t}^\phi(a|x,o)\dd a
\right]dx
\\
&=
\int\!\!\int
r_\beta(o,a)\psi(x)
p_{T|t}^\phi(a|x,o)p_t^\phi(x|o)
\dd a \dd x
\\
&=
\E_{\substack{
X_t\sim p_t^\phi(\cdot|o),\\
X_T\sim p_{T|t}^\phi(\cdot|X_t,o)
}}
\big[
r_\beta(o,X_T)\psi(X_t)
\big].
\end{aligned}
\label{eq:fm3}    
\end{align}

To derive the velocity reweighting identity, we use the marginalization
identity for the tilted flow:
\[
u^\beta(x,t,o)
=
\E_{X_T\sim p_{T|t}^\beta(\cdot|x,o)}
\big[
u^\beta(x|X_T,t,o)
\big].
\]
Multiplying both sides by \(p_t^\beta(x|o)\), we obtain
\[
\begin{aligned}
p_t^\beta(x|o)u^\beta(x,t,o)
&=
p_t^\beta(x|o)
\int
u^\beta(x|a,t,o)
p_{T|t}^\beta(a|x,o)\dd a
\\
&=
\int
u^\beta(x|a,t,o)
p_t^\beta(x|o)
p_{T|t}^\beta(a|x,o)\dd a .
\end{aligned}
\]
Using the ratio-tilted joint relation
\[
p_t^\beta(x|o)p_{T|t}^\beta(a|x,o)
=
\frac{r_\beta(o,a)}{Z_{\beta}(o)}
p_t^\phi(x|o)p_{T|t}^\phi(a|x,o),
\]
we get
\[
\begin{aligned}
p_t^\beta(x|o)u^\beta(x,t,o)
&=
\frac{1}{Z_{\beta}(o)}
p_t^\phi(x|o)
\int
r_\beta(o,a)
u^\beta(x|a,t,o)
p_{T|t}^\phi(a|x,o)\dd a
\\
&=
\frac{1}{Z_{\beta}(o)}
p_t^\phi(x|o)
\E_{X_T\sim p_{T|t}^\phi(\cdot|x,o)}
\big[
r_\beta(o,X_T)u^\beta(x|X_T,t,o)
\big].
\end{aligned}
\]
Equivalently,
\begin{align}
Z_{\beta}(o)p_t^\beta(x|o)u^\beta(x,t,o)
=
p_t^\phi(x|o)
\E_{X_T\sim p_{T|t}^\phi(\cdot|x,o)}
\big[
r_\beta(o,X_T)u^\beta(x|X_T,t,o)
\big].
\label{eq:fm2}
\end{align}
Now we convert \eqref{eq:fm2} into a test-function
identity. For any integrable test function \(h\), multiply
both sides of \eqref{eq:fm2} by
\(h(x,t,o)^\top\) and integrate over \(x\). This gives

\begin{align}
&Z_{\beta}(o)
\E_{X_t\sim p_t^\beta(\cdot|o)}
\big[
h(X_t,t,o)^\top u^\beta(X_t,t,o)
\big]
\notag \\
&=
Z_{\beta}(o)\int
h(x,t,o)^\top u^\beta(x,t,o)p_t^\beta(x|o)\dd x
\notag \\
&=
\int
h(x,t,o)^\top
p_t^\phi(x|o)
\E_{X_T\sim p_{T|t}^\phi(\cdot|x,o)}
\big[
r_\beta(o,X_T)u^\beta(x|X_T,t,o)
\big]
\dd x
\notag \\
&=
\int\!\!\int
r_\beta(o,a)
h(x,t,o)^\top u^\beta(x|a,t,o)
p_{T|t}^\phi(a|x,o)p_t^\phi(x|o)
\dd a\dd x
\notag \\
&=
\E_{\substack{
X_t\sim p_t^\phi(\cdot|o),\\
X_T\sim p_{T|t}^\phi(\cdot|X_t,o)
}}
\big[
r_\beta(o,X_T)
h(X_t,t,o)^\top u^\beta(X_t|X_T,t,o)
\big].\label{eq:fm1}
\end{align}

We now prove the gradient identity directly. Under the usual regularity
conditions allowing differentiation under the expectation,
\[
\begin{aligned}
\nabla_\theta \ell_{\mathrm{DM}}(\theta,o)
&=
2\E_{t,\,X_t\sim p_t^\beta(\cdot|o)}
\Big[
\big(\nabla_\theta u_\theta(X_t,t,o)\big)^\top
\big(
u_\theta(X_t,t,o)-u^\beta(X_t,t,o)
\big)
\Big].
\end{aligned}
\]
Multiplying both sides by \(Z_{\beta}(o)\), and applying
\eqref{eq:fm3} componentwise to
\[
\psi(x)
=
\big(\nabla_\theta u_\theta(x,t,o)\big)^\top
u_\theta(x,t,o),
\]
together with \eqref{eq:fm1} with
\(h=\nabla_\theta u_\theta\), gives
\begin{align}
&Z_{\beta}(o)\nabla_\theta \ell_{\mathrm{DM}}(\theta,o)\notag
\\
&=
2\E_{\substack{
t\sim U[0,T],\,
X_t\sim p_t^\phi(\cdot|o),\\
X_T\sim p_{T|t}^\phi(\cdot|X_t,o)
}}
\Big[
r_\beta(o,X_T)
\big(\nabla_\theta u_\theta(X_t,t,o)\big)^\top
\big(
u_\theta(X_t,t,o)-u^\beta(X_t|X_T,t,o)
\big)
\Big].
\label{eq:fm-dm-gradient-reweighted}
\end{align}

On the other hand, differentiating \(\ell_{\mathrm{CDM}}(\theta,o)\)
directly gives
\begin{align}
&\nabla_\theta \ell_{\mathrm{CDM}}(\theta,o)\notag \\
&=
2\E_{\substack{
t\sim U[0,T],\,
X_t\sim p_t^\phi(\cdot|o),\\
X_T\sim p_{T|t}^\phi(\cdot|X_t,o)
}}
\Big[
r_\beta(o,X_T)
\big(\nabla_\theta u_\theta(X_t,t,o)\big)^\top
\big(
u_\theta(X_t,t,o)-u^\beta(X_t|X_T,t,o)
\big)
\Big].\label{eq:fm-cdm-gradient}
\end{align}

Comparing \eqref{eq:fm-dm-gradient-reweighted} and
\eqref{eq:fm-cdm-gradient}, we obtain
\[
\nabla_\theta \ell_{\mathrm{CDM}}(\theta,o)
=
Z_{\beta}(o)\nabla_\theta \ell_{\mathrm{DM}}(\theta,o).
\]
\end{proof}

\clearpage
\section{\texorpdfstring{Proof of \cref{prop:cdm_solves_soc}}{Proof of Theorem \ref{prop:cdm_solves_soc}}}
\label{sec:proof_cdm_learn}

\begin{proposition}[\cref{prop:cdm_solves_soc} Restated: Conditional Drift Matching Solves the SOC Problem]
\label{prop:cdm_solves_soc_re}
Let $h_t$ be defined in \eqref{eq:solution_exp}.
Fix a state $o \in\mathcal S$ and let $\mathbb P^\phi(\cdot|o)$ be the reference path measure induced by the reverse-time SDE 
\begin{align*}
    \dd X_t
    =
    \bigl[
        f(X_t,t)-\sigma_t^2 s_\phi(X_t,t,o)
    \bigr]\dd t
    +
    \sigma_t\dd W_t,
\end{align*}
where $s_\phi(X_t,t,o)$ is the base model.
Consider the conditional drift matching objective
\begin{align*}
\mathcal{L}_{\mathrm{CDM}}(\theta)
=
\mathbb E_{t, X_t\sim p_t^\phi, X_0\sim p_{0|t}^\phi}
\left[
h_0(X_0,o)
\left\|
\nabla_x\log p_{0|t}^\phi(X_0|X_t,o)
-
s_\theta(X_t,t,o)
\right\|_2^2
\right].
\end{align*}
If the model class $\{s_\theta : \theta \in \Theta\}$ is realizable and the minimum of $L_{\mathrm{CDM}}$ is attained, then any minimizer $s_\theta^\star$ induces the optimal control 
$u_t^\star(x,o) := -s_\theta^\star(x,t,o)$
that solves the SOC problem
\begin{align}
    & ~ \quad \quad \quad
    \quad \quad \quad \quad \quad \inf_u 
    \frac12
    \mathbb E_{\mathbb Q^u(\cdot|o)}
    \int_0^T
    \|\sigma_t
    u_t(Y_t,o)\|_2^2\dd t, 
    \label{eq:oj_soc_re}
    \\
    \text{s.t.} & ~ 
    \dd Y_t
    =
    \bigl[
        f(Y_t,t)-\sigma_t^2 s_\phi(Y_t,t,o) + \sigma_t^2 u_t(Y_t,o)
    \bigr]\dd t
    +
    \sigma_t\dd W_t ,
    ~ Y_T \sim p_T^h, ~Y_0 \sim p_0^h,
    \label{eq:sde_soc_re}
\end{align}
where $\mathbb Q^u(\cdot|o)$ is the path measure induced by the controlled reverse-time SDE in \eqref{eq:sde_soc_re}, initialized at $t=T$ with $Y_T\sim p_T^h$, and $p_0^h,p_T^h$ are the endpoint distributions given by \eqref{eq:p_t^h}.
\end{proposition}

\begin{proof}
Fix $o \in\mathcal S$. For notational simplicity, the unconditional version is presented in this proof.
The conditional case follows identically by conditioning on $o$ throughout and replacing $p_t(\cdot),h_t(\cdot)$ with $p_t(\cdot|o),h_t(\cdot|o)$, respectively.

By \cref{thm:grad_equi_diff}, the drift matching objective has the same population minimizer as
\begin{align*}
    \mathcal{L}_{\mathrm{DM}}(\theta)
    =
    \E_{t,X_t\sim p_t^h}
    \left[
        \|
        \nabla_x\log h_t(X_t)-s_\theta(X_t,t)
        \|_2^2
    \right].
\end{align*}
Therefore, by realizability and attainment of the minimum, any minimizer satisfies
\begin{align*}
    s_\theta^\star(x,t)
    =
    \nabla_x\log h_t(x).
\end{align*}

It remains to identify the corresponding SOC optimizer.
We use the SOC theory in \cite{liu2025adjoint}.
Their formulation is written in forward time.
Our denoising process is written in reverse time.
Let
\begin{align*}
    \tau:=T-t,\qquad Z_\tau:=Y_{T-\tau}.
\end{align*}
Then the controlled process \eqref{eq:sde_soc_re} is written in forward time as
\begin{align*}
    \dd Z_\tau
    =
    \bigl[
        -f(Z_\tau,T-\tau)
        +
        \sigma_{T-\tau}^2s_\phi(Z_\tau,T-\tau)
        -
        \sigma_{T-\tau}^2u_{T-\tau}(Z_\tau)
    \bigr]\dd \tau
    +
    \sigma_{T-\tau}\dd W_\tau,
\end{align*}
where $Z_0\sim p_T^h,\; Z_T\sim p_0^h$.
For notation simplicity, we rewrite the process as
\begin{align*}
    \dd Z_\tau
    =
    [
        \bar f_\tau(Z_\tau)
        +
        \bar\sigma_\tau \bar u_\tau(Z_\tau)
    ]\dd \tau
    +
    \bar\sigma_\tau\dd W_\tau,
    \quad
    Z_0\sim p_T^h,\quad Z_T\sim p_0^h.
\end{align*}
where
$\bar\sigma_\tau :=
\sigma_{T-\tau},
\bar f_\tau(z)
:=
-f(z,T-\tau)
+
\sigma_{T-\tau}^2s_\phi(z,T-\tau),
\bar u_\tau(z)
:=
-\sigma_{T-\tau}u_{T-\tau}(z)$.

The associated SOC objective becomes
\begin{align*}
    \inf_{\bar u}
    \frac12
    \mathbb E
    \int_0^T
    \|
    \bar u_\tau(Z_\tau)
    \|_2^2\dd \tau,
\end{align*}
since
\begin{align*}
    \int_0^T
    \frac12
    \|
    \sigma_tu_t(Y_t)
    \|_2^2\dd t
    =
    \int_0^T
    \frac12
    \|
    \bar u_\tau(Z_\tau)
    \|_2^2\dd \tau.
\end{align*}

According to \cite{liu2025adjoint}, the optimal control satisfies
\begin{align*}
    \bar u_\tau^\star(z)
    =
    \bar\sigma_\tau\nabla_z\log \Phi_\tau(z),
\end{align*}
where the Schr\"odinger potentials satisfy
\begin{align}
    \Phi_\tau(z)
    &=
    \int
    \bar p_{T|\tau}^{\phi}(a|z)
    \Phi_T(a)\dd a,
    \qquad
    \Phi_0(z)\widehat\Phi_0(z)
    =
    p_T^h(z),
    \label{eq:soc_first_co}
    \\
    \widehat\Phi_\tau(z)
    &=
    \int
    \bar p_{\tau|0}^{\phi}(z|a)
    \widehat\Phi_0(a)\dd a,
    \qquad
    \Phi_T(z)\widehat\Phi_T(z)
    =
    p_0^h(z).
    \label{eq:soc_secon_co}
\end{align}
Here $\bar p_\tau^\phi$ denotes the marginal density of the forward-time process $Z_\tau=X_{T-\tau}$, and $\bar p_{b|a}^\phi$ denotes the conditional density.
Define
\begin{align*}
    A
    &:=
    \E_{X_0\sim p_0^\phi}[h_0(X_0)],
    \\
    \Phi_\tau(z)
    &:=
    h_{T-\tau}(z),
    \\
    \widehat\Phi_\tau(z)
    &:=
    \frac{p_{T-\tau}^\phi(z)}{A}.
\end{align*}
We now verify that these potentials satisfy the optimality equations.
First, for \eqref{eq:soc_first_co},
\begin{align*}
    \Phi_\tau(z)
    &=
    h_{T-\tau}(z)
    \\
    &=
    \E
    [
        h_0(X_0)
        | X_{T-\tau}=z
    ]
    \annot{By the definiton of $h_t$ following \cref{lem:algo_correctness}}
    \\
    &=
    \int
    \bar p_{T|\tau}^{\phi}(a|z)
    h_0(a)\dd a
    \annot{Rewrite the expectation as integral.}
    \\
    &=
    \int
    \bar p_{T|\tau}^{\phi}(a|z)
    \Phi_T(a)\dd a.
    \annot{By the definition of $\Phi_\tau = h_{T-\tau}$}
\end{align*}
Second, for \eqref{eq:soc_secon_co}, we have
\begin{align*}
    \widehat\Phi_\tau(z)
    &=
    \frac{p_{T-\tau}^\phi(z)}{A}
    \\
    &=
    \int
    \bar p_{\tau|0}^{\phi}(z|z_0)
    \frac{p_T^\phi(z_0)}{A}
    \dd z_0
    \annot{By the law of total probability}
    \\
    &=
    \int
    \bar p_{\tau|0}^{\phi}(z|z_0)
    \widehat\Phi_0(z_0)\dd z_0.
    \annot{By the definition of $\widehat \Phi_\tau$}
\end{align*}
Exploiting \eqref{eq:p_t^h}, we now verify the boundary conditions. 
At $\tau=0$,
\begin{align*}
    \Phi_0(z)\widehat\Phi_0(z)
    =
    h_T(z)
    \frac{p_T^\phi(z)}{A}
    =
    p_T^h(z).
\end{align*}
At $\tau=T$,
\begin{align*}
    \Phi_T(a)\widehat\Phi_T(a)
    =
    h_0(a)
    \frac{p_0^\phi(a)}{A}
    =
    p_0^h(a).
\end{align*}
Thus $\Phi_\tau(z)=h_{T-\tau}(z)$ satisfies the Schr\"odinger optimality equations.

Therefore,
\begin{align*}
    \bar u_\tau^\star(z)
    &=
    \bar\sigma_\tau\nabla_z\log \Phi_\tau(z)
    \\
    &=
    \sigma_{T-\tau}\nabla_z\log h_{T-\tau}(z).
\end{align*}
Since $\bar u_\tau(z)=-\sigma_{T-\tau}u_{T-\tau}(z)$, the optimal control in the original reverse-time convention is
\begin{align*}
    u_t^\star(x)
    =
    -\nabla_x\log h_t(x).
\end{align*}
Combining this with $s_\theta^\star(x,t)=\nabla_x\log h_t(x)$ gives
\begin{align*}
    u_\theta^\star(x,t)
    =
    -\nabla_x\log h_t(x)
    =
    -s_\theta^\star(x,t)
    .
\end{align*}
Hence any minimizer of $L_{\mathrm{CDM}}$ induces the optimal control for the SOC problem.
\end{proof}

\section{Implementation Details}
\label{sec:implementation}

\subsection{Policy Evaluation}

\paragraph{Policy Evaluation for DoobPPO.}
The policy evaluation of DoobPPO learns a state-value function by minimizing the regression loss
\begin{equation}
\mathcal{L}^{V}(\psi)
=
\mathbb{E}_{o_t \sim \mathcal{D}}
\Big[
\big(V_\psi(o_t)-\hat{R}_t\big)^2
\Big],
\end{equation}
where $\hat{R}_t$ is the empirical return or bootstrap target computed from rollout data.
Following the standard PPO-style evaluation, we estimate the advantage with generalized advantage estimation,
\begin{equation}
A_t=\sum_{l=0}^{T-t-1}(\gamma\lambda)^l \delta_{t+l},
\qquad
\delta_t=r_t+\gamma V_\psi(o_{t+1})-V_\psi(o_t).
\end{equation}
In practice, the resulting advantage estimator is used as the policy improvement signal for training the ratio network in the first stage.

\paragraph{Policy Evaluation for DoobGRPO.}
The policy evaluation of DoobGRPO computes group-relative advantages from rollout returns.
For each state, we sample a group of candidate actions or trajectories under the current diffusion policy, evaluate their returns, and normalize them within the group.
The resulting relative advantage is given by
\begin{equation}
A_i^{\mathrm{GRPO}}
=
\frac{R_i-\mathrm{mean}(\{R_j\}_{j=1}^{G})}
{\mathrm{std}(\{R_j\}_{j=1}^{G})+\epsilon},
\end{equation}
where $R_i$ is the return of the $i$-th sample and $G$ is the group size.
In practice, these normalized relative advantages are directly used as the policy improvement signal for the first-stage ratio learning objective.

\subsection{Unbiased Ratio Regularizer}

The ratio regularizer
$\E_{o\sim\mathcal D}\left[\left(\E_{a'\sim \pi_{\mathrm{old}}(\cdot|o)}[r_\beta(o,a')]-1\right)^2\right]$
contains a quadratic moment term, whose naive Monte Carlo estimate is biased.
To obtain an unbiased estimator, we use double sampling.
For each sampled state $o$, two independent action batches are drawn from $\pi_{\mathrm{old}}(\cdot|o)$, yielding two conditionally independent estimates $\hat g_1(o)$ and $\hat g_2(o)$ of
$g(o):=\E_{a'\sim\pi_{\mathrm{old}}(\cdot|o)}[r_\beta(o,a')]-1$.
We then replace $g(o)^2$ with $\hat g_1(o)\hat g_2(o)$, which satisfies
$\E[\hat g_1(o)\hat g_2(o)|o]=g(o)^2$.

\subsection{Ratio Regularizer in the Off-Policy Setting}

In the on-policy variant, the ratio centering regularizer enforces
$\mathbb{E}_{a' \sim \pi_{\mathrm{old}}(\cdot|o)}[r_\beta(o,a')] \approx 1$
by sampling actions from the frozen rollout policy, which corresponds exactly to the normalization constraint of the density ratio.
In the off-policy variant with a replay buffer, there is no single well-defined behavior policy associated with each stored transition, since the buffer aggregates experience from many past policy versions.
We therefore estimate the centering constraint by sampling $a'$ from the current actor at the same state~$o$.
This ensures that the learned ratio remains well-calibrated within the support of the evolving policy, effectively regularizing $r_\beta$ against degenerate solutions regardless of the age of the replay sample.
Empirically, this regularizer is applied with a small coefficient ($\lambda = 0.01$) and we observe that it is sufficient to stabilize ratio learning throughout training.

\subsection{Computation and Memory Cost}

We report the GPU memory allocation and total training time in Table~\ref{tab:computational_cost}.
All measurements are obtained on a desktop workstation with an Intel Xeon Gold 6242R CPU, 376 GB memory, and an NVIDIA A6000 GPU.
Although DoobPPO and DoobGRPO are not the most lightweight methods in terms of raw memory or wall-clock time, their computational overhead remains moderate and is comparable to several strong diffusion-policy baselines.
In particular, DoobPPO requires 1578 MB GPU memory and 25.28 minutes, while DoobGRPO requires 1795 MB and 28.12 minutes, both remaining in the same practical regime as DPMD, DACER, and SDAC.
Considering that our methods achieve state-of-the-art or near-state-of-the-art performance on most benchmark tasks, these results indicate a favorable performance--efficiency trade-off.
Overall, DoobPO delivers stronger empirical performance without incurring prohibitive computational or memory cost.

\begin{table*}[h]
\centering
\caption{\small GPU memory allocation and total compute time of 200K
iterations and 1 million environment interactions.}
\vspace{-2mm}
\label{tab:computational_cost}
\resizebox{0.8\linewidth}{!}{%
\begin{tabular}{r|c|c}
\toprule
Algorithm
& GPU Memory (MB)
& Training Time (min)
\\
\midrule
DIPO
& 5212
& 15.96
\\

DPPO
& 1087
& 12.71
\\

DPMD
& 1175
& 25.29
\\

DACER
& 1312
& 29.12
\\

SDAC
& 1206
& 23.52
\\

\textbf{DoobPPO}
& 1578
& 25.28
\\

\textbf{DoobGRPO}
& 1795
& 28.12
\\

\bottomrule
\end{tabular}}
\end{table*}

\subsection{Hyperparameters}

\begin{table*}[h]
\centering
\caption{\small Hyperparameters}
\vspace{-2mm}
\label{tab:hyperparameters}
\resizebox{0.8\linewidth}{!}{%
\begin{tabular}{l|c}
\toprule
Name
& Value
\\
\midrule
Critic learning rate
& 3e-4
\\

Policy learning rate
& 3e-4, linear annealing to 3e-5
\\

Diffusion steps
& 20
\\

Diffusion noise schedules
& Linear
\\

Policy network hidden layers
& 3
\\

Policy network hidden neurons
& 256
\\

Policy network activation
& Mish
\\

Ratio network hidden layers
& 3
\\

Ratio network activation
& Mish
\\

Replay buffer size
& 1 million
\\

\bottomrule
\end{tabular}}
\end{table*}

\begin{table*}[h]
\centering
\caption{\small DoobPPO hyperparameters on Robomimic (\texttt{Can} and \texttt{Square}).}
\vspace{-2mm}
\label{tab:hyperparameters_robomimic}
\resizebox{0.8\linewidth}{!}{%
\begin{tabular}{l|c}
\toprule
Name
& Value
\\
\midrule
Discount factor $\gamma$
& 0.999
\\

Actor learning rate
& $2\times10^{-5}$
\\

Ratio-network learning rate
& $1\times10^{-4}$
\\

Reward scale
& 2.5
\\

PPO clipping coefficient
& 0.2
\\

Maximum ratio weight
& 2
\\

Ratio regularization coefficient
& 0.03
\\

Action-chunk horizon
& 4
\\

Training iterations
& 151 (\texttt{Can}), 200 (\texttt{Square})
\\

Parallel evaluation environments
& 50
\\

\bottomrule
\end{tabular}}
\end{table*}

\section{Ablation Studies}
\label{sec:ablation}

\begin{figure}[h]
    \centering
    \begin{subfigure}{0.49\linewidth}
        \centering
        \includegraphics[width=\linewidth]{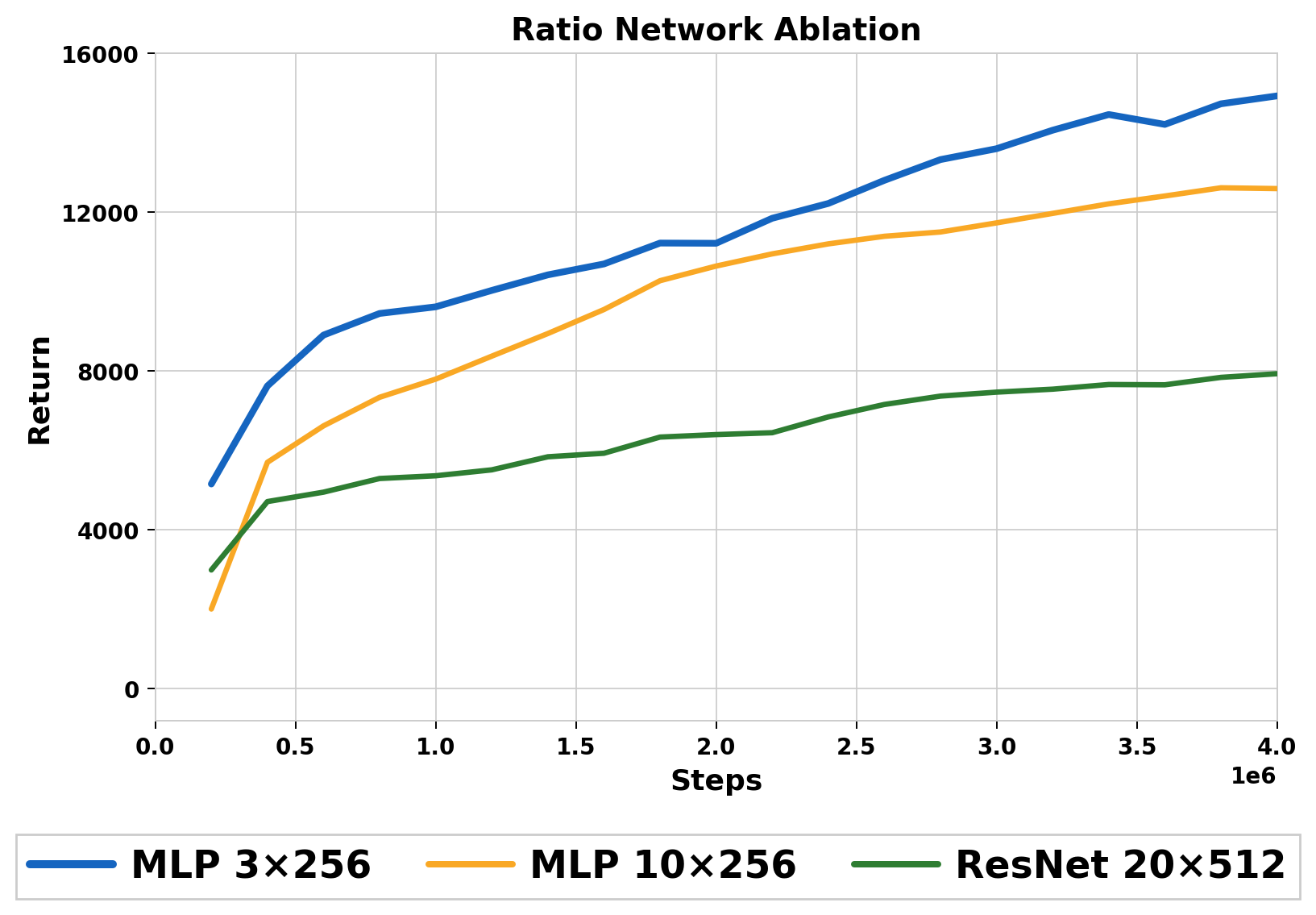}
        \caption{\small Effect of ratio network architecture.}
    \end{subfigure}
    \hfill
    \begin{subfigure}{0.49\linewidth}
        \centering
        \includegraphics[width=\linewidth]{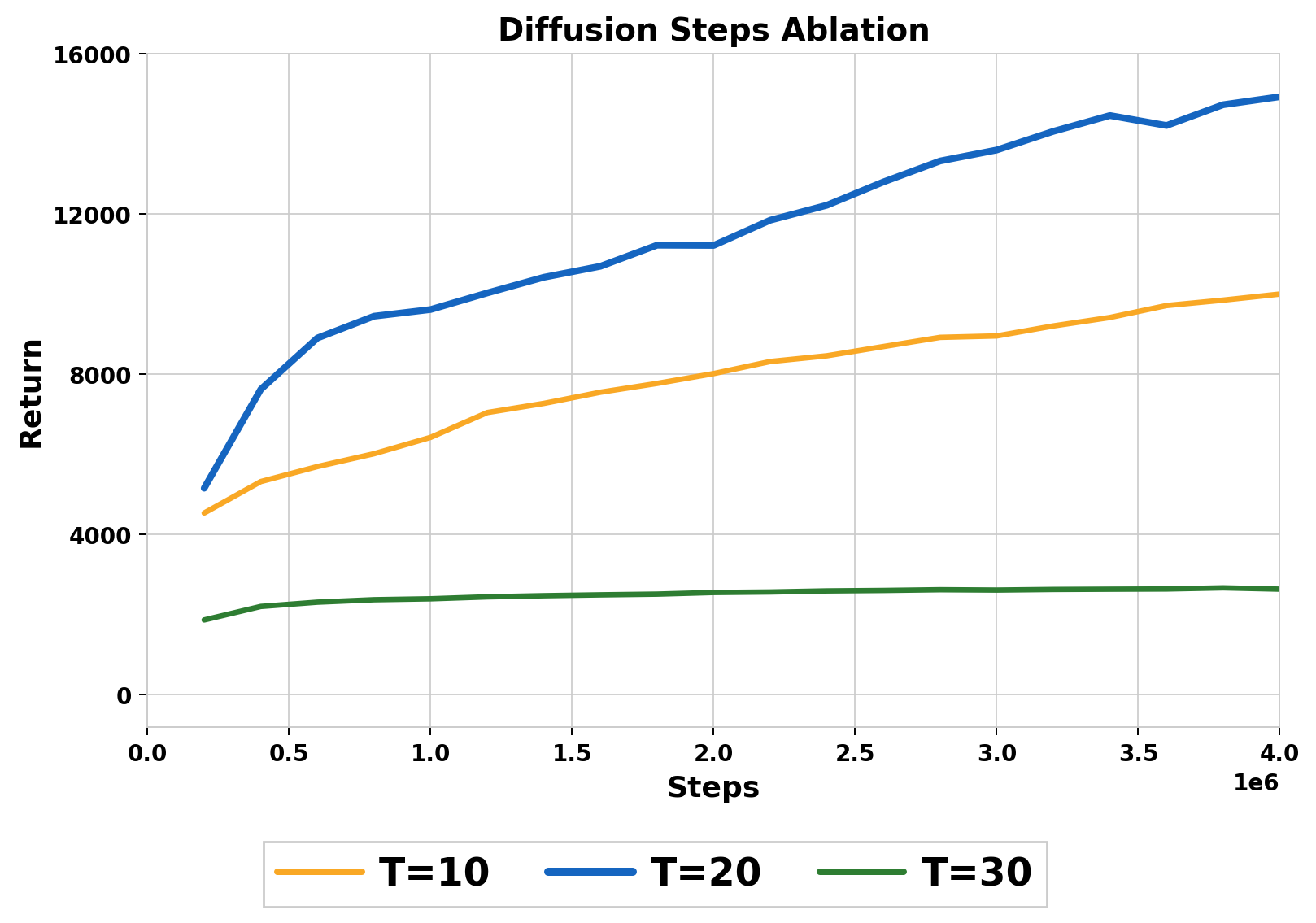}
        \caption{\small Effect of the number of diffusion steps.}
    \end{subfigure}

    \vspace{-2mm}
    \caption{{\small Ablation study of DoobPPO on OpenAI Gym MuJoCo \texttt{HalfCheetah-v4}.
    \textbf{(a)} Effect of the ratio network architecture on policy performance.
    \textbf{(b)} Different diffusion steps lead to significantly different performance, with moderate steps performing best.}}
    \label{fig:ablation_two}
\end{figure}

\subsection{Diffusion Policy.}

\paragraph{Ablation Study on Ratio Network Size.}
We study the sensitivity of DoobPPO to the ratio-network architecture on \texttt{HalfCheetah-v4}.
We compare three variants: an MLP ratio network with 3 layers of width 256, a deeper MLP with 10 layers of width 256, and a larger ResNet-style architecture with 20 layers of width 512.
As shown in Figure~\ref{fig:ablation_two}, the simple 3$\times$256 MLP achieves the best performance, consistently outperforming both the deeper MLP and the larger ResNet-style model.
This suggests that increasing the capacity of the ratio network does not necessarily improve performance, and may even make optimization less effective.
Based on this result, we use the 3$\times$256 MLP ratio network in all diffusion-policy experiments.

\paragraph{Ablation Study on Diffusion Steps.}
We also study the sensitivity of DoobPPO to the number of diffusion steps on \texttt{HalfCheetah-v4}.
We compare three settings, namely $T=10$, $T=20$, and $T=30$.
The results show that $T=20$ achieves the best overall performance, while $T=10$ remains competitive but consistently underperforms $T=20$, and $T=30$ leads to a substantial degradation.
These results suggest that DoobPO benefits from a moderately sized diffusion process rather than a very shallow or overly long denoising chain.
Therefore, we use $20$ diffusion steps in all diffusion-policy experiments.

\subsection{Flow Policy.}

\begin{figure}[h]
    \centering
    \begin{subfigure}{0.49\linewidth}
        \centering
        \includegraphics[width=\linewidth]{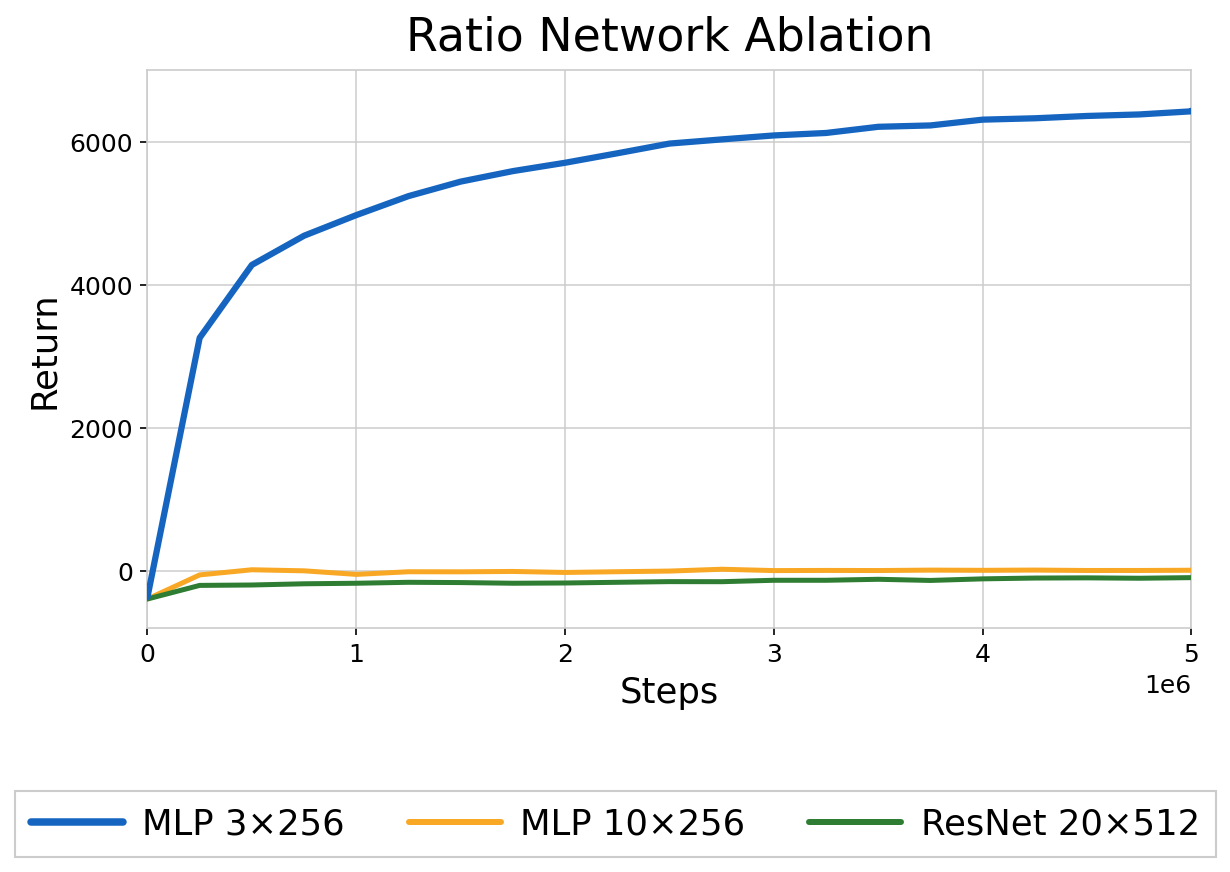}
        \caption{\small Effect of ratio network architecture.}
    \end{subfigure}
    \hfill
    \begin{subfigure}{0.49\linewidth}
        \centering
        \includegraphics[width=\linewidth]{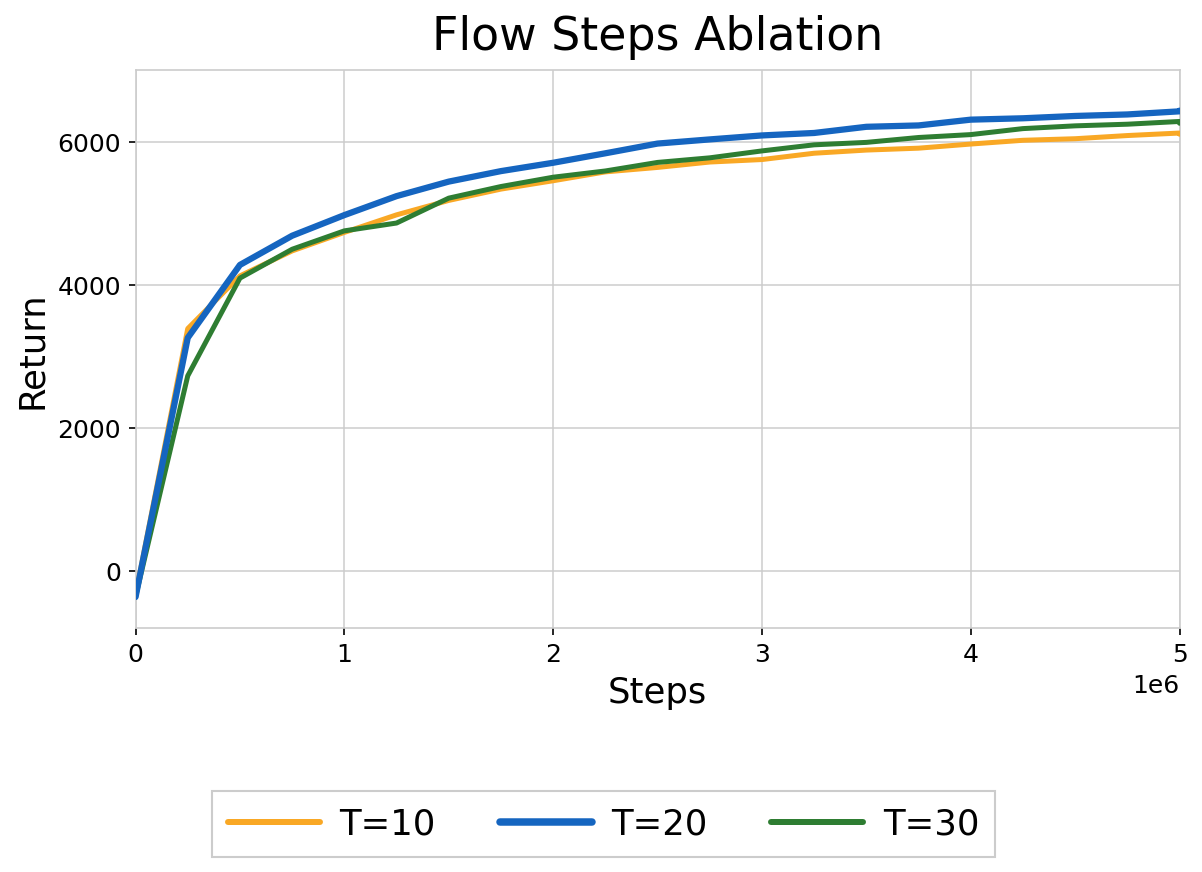}
        \caption{\small Effect of the number of flow steps.}
    \end{subfigure}

    \vspace{-2mm}
    \caption{{\small Ablation study of DoobGRPO flow policy on OpenAI Gym MuJoCo \texttt{HalfCheetah-v4}.
    \textbf{(a)} Only the lightweight MLP 3$\times$256 ratio network learns effectively, while deeper and larger architectures fail to train.
    \textbf{(b)} Performance is robust to the number of flow steps, with $T{=}20$ achieving the best result and $T{=}10$, $T{=}30$ remaining competitive.}}
    \label{fig:ablation_flow}
\end{figure}

\paragraph{Ablation Study on Ratio Network Size.}
We study the sensitivity of DoobGRPO to the ratio-network architecture on \texttt{HalfCheetah-v4} with flow policies.
We compare the same three variants as in the diffusion-policy setting: an MLP ratio network with 3 layers of width 256, a deeper MLP with 10 layers of width 256, and a ResNet-style architecture with 20 blocks of width 512.
As shown in Figure~\ref{fig:ablation_flow}(a), the simple 3$\times$256 MLP achieves the best performance with a peak return of 6433, while both the deeper MLP (peak return 23) and the larger ResNet (peak return $-$94) fail to learn meaningful policies.
This is consistent with the diffusion-policy ablation, confirming that a lightweight ratio network is critical for stable training across both generative-model backbones.

\paragraph{Ablation Study on Flow Steps.}
We also study the sensitivity of DoobGRPO to the number of flow integration steps on \texttt{HalfCheetah-v4}.
We compare $T=10$, $T=20$, and $T=30$.
As shown in Figure~\ref{fig:ablation_flow}(b), the default setting $T=20$ achieves the best peak performance (6433), while $T=10$ (6120) and $T=30$ (6282) remain competitive but slightly underperform $T=20$.
These results suggest that DoobPO is relatively robust to the choice of flow steps, with $T=20$ providing the best trade-off between sampling quality and computational cost.
We use $20$ flow steps in all flow-policy experiments.

\clearpage
\section{Algorithms}

\subsection{Diffusion Policy Optimization with DoobPO}

\begin{algorithm}[h]
\caption{Diffusion Policy Optimization with DoobPO}
\label{alg:rwdm_ppo}
\begin{algorithmic}[1]
\Require Environment MDP $\mathcal{M}=(\mathcal{S},\mathcal{A},P,R,\gamma)$;
pre-trained score model $g_{\phi}$;
reinforcement learning objective $J_{\mathrm{RL}}$;
advantage estimator $\hat{A}$;
ratio network $r_\beta$;
ratio regularization coefficient $\lambda_{\mathrm{reg}}$;
iteration $K$
\State \textbf{Initialize:} drift model $\theta_1=\phi$; ratio net $\beta_1$.
\For{$k=1,2,\ldots,K$}
    \State Add rollout samples to replay buffer $\mathcal{D}$ from $\mathcal{M}$ using current policy $\pi^k$ induced by $g_{\theta_k}$
    \State Estimate advantage of current policy $A^{\pi^k}\leftarrow \hat{A}(\mathcal{D})$

    \State \textbf{Ratio learning update}
    \State Update $\beta$ by optimizing the regularized reinforcement learning objective
    \begin{align*}
        \beta_{k+1}
        \leftarrow
        \argmax_{\beta}
        \Big\{
        J_{\mathrm{RL}}(\beta)
        -
        \lambda_{\mathrm{reg}}
        \E_{o\sim\mathcal D}
        \Big[
        \big(
        \E_{a\sim \pi^k(\cdot|o)}[r_\beta(o,a)]-1
        \big)^2
        \Big]
        \Big\}.
    \end{align*}

    \State \textbf{Doob's drift matching update}
    \State For each $(o,a)\in\mathcal{D}$, sample $t\sim U[0,T]$ and generate noised action $x_t$
    \begin{align*}
        x_t \leftarrow \sqrt{\bar\alpha_t}\,a
    + \sqrt{1-\bar\alpha_t}\,z,
\quad z\sim\mathcal N(0,I)
    \end{align*}
    \State Update $\theta$ by minimizing the conditional drift matching objective
    \begin{align*}
        \theta_{k+1}\leftarrow \argmin_{\theta}
        \E_{(o,a,x_t,t)\sim\mathcal{D}}
        \Big[
        r_{\beta_{k+1}}(o,a)
        \Big\|
        -\frac{x_t-\sqrt{\bar\alpha_t}\,a}
    {1-\bar\alpha_t}
    -g_\theta(x_t,t,o)
        \Big\|_2^2
        \Big]
    \end{align*}
\EndFor
\State \textbf{Output:} Policy $\pi^{K+1}$ with learned drift $g_{\theta_{K+1}}$
\end{algorithmic}
\end{algorithm}

\clearpage
\subsection{Flow Policy Optimization with DoobPO}

\begin{algorithm}[h]
\caption{Flow Policy Optimization with DoobPO}
\label{alg:rwdm_ppo_fm}
\begin{algorithmic}[1]
\Require Environment MDP $\mathcal{M}=(\mathcal{S},\mathcal{A},P,R,\gamma)$;
pre-trained velocity model $u_{\phi}$;
reinforcement learning objective $J_{\mathrm{RL}}$;
advantage estimator $\hat{A}$;
ratio network $r_\beta$;
ratio regularization coefficient $\lambda_{\mathrm{reg}}$;
iteration $K$
\State \textbf{Initialize:} velocity model $\theta_1=\phi$; ratio net $\beta_1$.
\For{$k=1,2,\ldots,K$}
    \State Add rollout samples to replay buffer $\mathcal{D}$ from $\mathcal{M}$ using current policy $\pi^k$ induced by $u_{\theta_k}$
    \State Estimate advantage of current policy $A^{\pi^k}\leftarrow \hat{A}(\mathcal{D})$

    \State \textbf{Ratio learning update}
    \State Update $\beta$ by optimizing the regularized reinforcement learning objective
    \begin{align*}
        \beta_{k+1}
        \leftarrow
        \argmax_{\beta}
        \Big\{
        J_{\mathrm{RL}}(\beta)
        -
        \lambda_{\mathrm{reg}}
        \E_{o\sim\mathcal D}
        \Big[
        \big(
        \E_{a\sim \pi^k(\cdot|o)}[r_\beta(o,a)]-1
        \big)^2
        \Big]
        \Big\}.
    \end{align*}

    \State \textbf{Drift matching update}
    \State For each $(o,a)\in\mathcal{D}$, sample $t \sim U[0,T]$ and generate intermediate action $x_t$ by
    \begin{align*}
        x_t \leftarrow \alpha_t z
        + \beta_t a,
        \quad z \sim N(0,I)
    \end{align*}
    \State Update $\theta$ by minimizing the conditional drift matching objective
    \begin{align*}
        \theta_{k+1} \leftarrow \argmin_{\theta}
        \E_{(o,a,x_t,t)\sim \mathcal{D}}
        \Big[ r_{\beta_{k+1}}(o,a)\,
        \|
            \dot{\alpha}_t z + \dot{\beta}_t a
            - u_\theta(x_t,t,o)
        \|_2^2
        \Big]
    \end{align*}
\EndFor
\State \textbf{Output:} Policy $\pi^{K+1}$ with learned velocity $u_{\theta_{K+1}}$
\end{algorithmic}
\end{algorithm}

\clearpage
\def\arxivfont{\rm}
\bibliographystyle{plainnat}

\bibliography{refs}

@article{lipman2024flow,
  title={Flow matching guide and code},
  author={Lipman, Yaron and Havasi, Marton and Holderrieth, Peter and Shaul, Neta and Le, Matt and Karrer, Brian and Chen, Ricky TQ and Lopez-Paz, David and Ben-Hamu, Heli and Gat, Itai},
  journal={arXiv preprint arXiv:2412.06264},
  year={2024}
}

@article{song2020score,
  title={Score-based generative modeling through stochastic differential equations},
  author={Song, Yang and Sohl-Dickstein, Jascha and Kingma, Diederik P and Kumar, Abhishek and Ermon, Stefano and Poole, Ben},
  journal={arXiv preprint arXiv:2011.13456},
  year={2020}
}

@book{sutton1998reinforcement,
  title={Reinforcement learning: An introduction},
  author={Sutton, Richard S and Barto, Andrew G and others},
  year={1998},
  publisher={MIT press Cambridge}
}

@article{schulman2017proximal,
  title={Proximal policy optimization algorithms},
  author={Schulman, John and Wolski, Filip and Dhariwal, Prafulla and Radford, Alec and Klimov, Oleg},
  journal={arXiv preprint arXiv:1707.06347},
  year={2017}
}

@article{ren2024diffusion,
  title={Diffusion policy policy optimization},
  author={Ren, Allen Z and Lidard, Justin and Ankile, Lars L and Simeonov, Anthony and Agrawal, Pulkit and Majumdar, Anirudha and Burchfiel, Benjamin and Dai, Hongkai and Simchowitz, Max},
  journal={arXiv preprint arXiv:2409.00588},
  year={2024}
}

@article{chi2025diffusion,
  title={Diffusion policy: Visuomotor policy learning via action diffusion},
  author={Chi, Cheng and Xu, Zhenjia and Feng, Siyuan and Cousineau, Eric and Du, Yilun and Burchfiel, Benjamin and Tedrake, Russ and Song, Shuran},
  journal={The International Journal of Robotics Research},
  volume={44},
  number={10-11},
  pages={1684--1704},
  year={2025},
  publisher={Sage Publications Sage UK: London, England}
}

@article{intelligence2025pi,
  title={$\pi_{0.6}^*$: a VLA That Learns From Experience},
  author={Intelligence, Physical and Amin, Ali and Aniceto, Raichelle and Balakrishna, Ashwin and Black, Kevin and Conley, Ken and Connors, Grace and Darpinian, James and Dhabalia, Karan and DiCarlo, Jared and others},
  journal={arXiv preprint arXiv:2511.14759},
  year={2025}
}

@article{park2025flow,
  title={Flow q-learning},
  author={Park, Seohong and Li, Qiyang and Levine, Sergey},
  journal={arXiv preprint arXiv:2502.02538},
  year={2025}
}

@article{zhang2025reinflow,
  title={ReinFlow: Fine-tuning flow matching policy with online reinforcement learning},
  author={Zhang, Tonghe and Yu, Chao and Su, Sichang and Wang, Yu},
  journal={arXiv preprint arXiv:2505.22094},
  year={2025}
}

@article{black2024pi,
  title={$\pi_0$: A Vision-Language-Action Flow Model for General Robot Control},
  author={Black, Kevin and Brown, Noah and Driess, Danny and Esmail, Adnan and Equi, Michael and Finn, Chelsea and Fusai, Niccolo and Groom, Lachy and Hausman, Karol and Ichter, Brian and others},
  journal={CoRR},
  year={2024}
}

@article{lipman2022flow,
  title={Flow matching for generative modeling},
  author={Lipman, Yaron and Chen, Ricky TQ and Ben-Hamu, Heli and Nickel, Maximilian and Le, Matt},
  journal={arXiv preprint arXiv:2210.02747},
  year={2022}
}

@article{williams1992simple,
  title={Simple statistical gradient-following algorithms for connectionist reinforcement learning},
  author={Williams, Ronald J},
  journal={Machine learning},
  volume={8},
  number={3},
  pages={229--256},
  year={1992},
  publisher={Springer}
}

@article{shao2024deepseekmath,
  title={Deepseekmath: Pushing the limits of mathematical reasoning in open language models},
  author={Shao, Zhihong and Wang, Peiyi and Zhu, Qihao and Xu, Runxin and Song, Junxiao and Bi, Xiao and Zhang, Haowei and Zhang, Mingchuan and Li, YK and Wu, Yang and others},
  journal={arXiv preprint arXiv:2402.03300},
  year={2024}
}

@inproceedings{schulman2015trust,
  title={Trust region policy optimization},
  author={Schulman, John and Levine, Sergey and Abbeel, Pieter and Jordan, Michael and Moritz, Philipp},
  booktitle={International conference on machine learning},
  pages={1889--1897},
  year={2015},
  organization={PMLR}
}

@article{ma2025efficient,
  title={Efficient Online Reinforcement Learning for Diffusion Policy},
  author={Ma, Haitong and Chen, Tianyi and Wang, Kai and Li, Na and Dai, Bo},
  journal={arXiv preprint arXiv:2502.00361},
  year={2025}
}

@inproceedings{haarnoja2018soft,
  title={Soft actor-critic: Off-policy maximum entropy deep reinforcement learning with a stochastic actor},
  author={Haarnoja, Tuomas and Zhou, Aurick and Abbeel, Pieter and Levine, Sergey},
  booktitle={International conference on machine learning},
  pages={1861--1870},
  year={2018},
  organization={Pmlr}
}

@article{zhan2023policy,
  title={Policy mirror descent for regularized reinforcement learning: A generalized framework with linear convergence},
  author={Zhan, Wenhao and Cen, Shicong and Huang, Baihe and Chen, Yuxin and Lee, Jason D and Chi, Yuejie},
  journal={SIAM Journal on Optimization},
  volume={33},
  number={2},
  pages={1061--1091},
  year={2023},
  publisher={SIAM}
}

@article{song2019generative,
  title={Generative modeling by estimating gradients of the data distribution},
  author={Song, Yang and Ermon, Stefano},
  journal={Advances in neural information processing systems},
  volume={32},
  year={2019}
}

@article{ho2020denoising,
  title={Denoising diffusion probabilistic models},
  author={Ho, Jonathan and Jain, Ajay and Abbeel, Pieter},
  journal={Advances in neural information processing systems},
  volume={33},
  pages={6840--6851},
  year={2020}
}

@article{liu2022flow,
  title={Flow straight and fast: Learning to generate and transfer data with rectified flow},
  author={Liu, Xingchao and Gong, Chengyue and Liu, Qiang},
  journal={arXiv preprint arXiv:2209.03003},
  year={2022}
}

@article{ke20243d,
  title={3d diffuser actor: Policy diffusion with 3d scene representations},
  author={Ke, Tsung-Wei and Gkanatsios, Nikolaos and Fragkiadaki, Katerina},
  journal={arXiv preprint arXiv:2402.10885},
  year={2024}
}

@article{huang2024diffuseloco,
  title={Diffuseloco: Real-time legged locomotion control with diffusion from offline datasets},
  author={Huang, Xiaoyu and Chi, Yufeng and Wang, Ruofeng and Li, Zhongyu and Peng, Xue Bin and Shao, Sophia and Nikolic, Borivoje and Sreenath, Koushil},
  journal={arXiv preprint arXiv:2404.19264},
  year={2024}
}

@article{psenka2023learning,
  title={Learning a diffusion model policy from rewards via q-score matching},
  author={Psenka, Michael and Escontrela, Alejandro and Abbeel, Pieter and Ma, Yi},
  journal={arXiv preprint arXiv:2312.11752},
  year={2023}
}

@article{ding2024diffusion,
  title={Diffusion-based reinforcement learning via q-weighted variational policy optimization},
  author={Ding, Shutong and Hu, Ke and Zhang, Zhenhao and Ren, Kan and Zhang, Weinan and Yu, Jingyi and Wang, Jingya and Shi, Ye},
  journal={Advances in Neural Information Processing Systems},
  volume={37},
  pages={53945--53968},
  year={2024}
}

@article{mcallister2025flow,
  title={Flow matching policy gradients},
  author={McAllister, David and Ge, Songwei and Yi, Brent and Kim, Chung Min and Weber, Ethan and Choi, Hongsuk and Feng, Haiwen and Kanazawa, Angjoo},
  journal={arXiv preprint arXiv:2507.21053},
  year={2025}
}

@article{wang2024diffusionactorcriticentropyregulator,
  title={Diffusion actor-critic with entropy regulator},
  author={Wang, Yinuo and Wang, Likun and Jiang, Yuxuan and Zou, Wenjun and Liu, Tong and Song, Xujie and Wang, Wenxuan and Xiao, Liming and Wu, Jiang and Duan, Jingliang and others},
  journal={Advances in Neural Information Processing Systems},
  volume={37},
  pages={54183--54204},
  year={2024}
}

@article{yang2023policyrepresentationdiffusionprobability,
  title={Policy representation via diffusion probability model for reinforcement learning},
  author={Yang, Long and Huang, Zhixiong and Lei, Fenghao and Zhong, Yucun and Yang, Yiming and Fang, Cong and Wen, Shiting and Zhou, Binbin and Lin, Zhouchen},
  journal={arXiv preprint arXiv:2305.13122},
  year={2023}
}

@inproceedings{sohldickstein2015deep,
  title={Deep Unsupervised Learning using Nonequilibrium Thermodynamics},
  author={Sohl-Dickstein, Jascha and Weiss, Eric and Maheswaranathan, Niru and Ganguli, Surya},
  booktitle={ICML},
  year={2015}
}

@article{song2021scorebased,
  title={Score-Based Generative Modeling through Stochastic Differential Equations},
  author={Song, Yang and Sohl-Dickstein, Jascha and Kingma, Diederik P. and Kumar, Abhishek and Ermon, Stefano and Poole, Ben},
  journal={ICLR},
  year={2021}
}

@article{chen2025pirl,
  title={$\pi$RL: Online rl fine-tuning for flow-based vision-language-action models},
  author={Chen, Kang and Liu, Zhihao and Zhang, Tonghe and Guo, Zhen and Xu, Si and Lin, Hao and Zang, Hongzhi and Zhang, Quanlu and Yu, Zhaofei and Fan, Guoliang and others},
  journal={arXiv preprint arXiv:2510.25889},
  year={2025}
}

@article{reuss2023goalconditionedimitationlearningusing,
  title={Goal-conditioned imitation learning using score-based diffusion policies},
  author={Reuss, Moritz and Li, Maximilian and Jia, Xiaogang and Lioutikov, Rudolf},
  journal={arXiv preprint arXiv:2304.02532},
  year={2023}
}

@inproceedings{chen2023playfusionskillacquisitiondiffusion,
  title={Playfusion: Skill acquisition via diffusion from language-annotated play},
  author={Chen, Lili and Bahl, Shikhar and Pathak, Deepak},
  booktitle={Conference on Robot Learning},
  pages={2012--2029},
  year={2023},
  organization={PMLR}
}

@inproceedings{nguyen2025h,
  title={h-Edit: Effective and Flexible Diffusion-Based Editing via Doob's h-Transform},
  author={Nguyen, Toan and Do, Kien and Kieu, Duc and Nguyen, Thin},
  booktitle={Proceedings of the Computer Vision and Pattern Recognition Conference},
  pages={28490--28501},
  year={2025}
}

@article{zhu2026training,
  title={Training-Free Adaptation of Diffusion Models via Doob's $ h $-Transform},
  author={Zhu, Qijie and Ye, Zeqi and Liu, Han and Wang, Zhaoran and Chen, Minshuo},
  journal={arXiv preprint arXiv:2602.16198},
  year={2026}
}

@article{liu2025adjoint,
  title={Adjoint schr{\"o}dinger bridge sampler},
  author={Liu, Guan-Horng and Choi, Jaemoo and Chen, Yongxin and Miller, Benjamin K and Chen, Ricky TQ},
  journal={Advances in Neural Information Processing Systems},
  volume={38},
  pages={15673--15708},
  year={2025}
}

@article{mandlekar2021matters,
  title={What Matters in Learning from Offline Human Demonstrations for Robot Manipulation},
  author={Mandlekar, Ajay and Xu, Danfei and Wong, Josiah and Nasiriany, Soroush and Wang, Chen and Kulkarni, Rohun and Fei-Fei, Li and Savarese, Silvio and Zhu, Yuke and Mart{\'\i}n-Mart{\'\i}n, Roberto},
  journal={arXiv preprint arXiv:2108.03298},
  year={2021}
}

\end{document}